\RequirePackage{fix-cm}

\documentclass[twocolumn]{svjour3}          
\smartqed  

\usepackage{graphicx}
\usepackage{mathptmx}      

\SetSymbolFont{letters}{bold}{OML}{cmm}{b}{it}      
\SetSymbolFont{operators}{bold}{OT1}{cmr}{bx}{n}    

\usepackage{epsfig}
\usepackage{epstopdf}
\usepackage{amssymb,amsmath} 
\usepackage[lofdepth,lotdepth]{subfig}
\usepackage[table]{xcolor}
\usepackage{url}
\usepackage{verbatim}
\usepackage[mathscr]{eucal}
\usepackage{latexsym}
\usepackage{algorithmic}
\usepackage{algorithm}
\usepackage{multirow}
\usepackage{float}
\usepackage{rotating}
\usepackage{caption}

\newcommand {\bsigma}{\boldsymbol{\sigma}} 
\newcommand {\bSigma}{\boldsymbol{\Sigma}} 
\newcommand {\bpsi}{\boldsymbol{\psi}}     
\newcommand {\A}{\boldsymbol{A}}            
\newcommand {\U}{\boldsymbol{U}}           
\newcommand {\V}{\boldsymbol{V}}           
\newcommand {\M}{\boldsymbol{M}}           
\newcommand {\D}{\boldsymbol{D}}           
\newcommand {\bx}{\boldsymbol{x}}          
\newcommand {\bv}{\boldsymbol{v}}          
\newcommand {\bu}{\boldsymbol{u}}          
\newcommand {\R}{\mathbb {R}}              
\newcommand {\N}{\mathbb {N}}              
\newcommand {\QED}{\begin{flushright} $\square$ \end{flushright}}   
\DeclareMathOperator{\Dim}{dim}
\DeclareMathOperator{\Span}{span}

\journalname{International Journal of Computer Vision}

\begin{document}

\title{A New Approach To Two-View Motion Segmentation Using Global Dimension Minimization}

\author{Bryan Poling \and Gilad Lerman}

\institute{Bryan Poling \at
           206 Church St. SE, Minneapolis, MN 55455 \\
           Tel.: +612-625-5099\\
           \email{poli0048@math.umn.edu}
           \and
           Gilad Lerman \at
           206 Church St. SE, Minneapolis, MN 55455 \\
           Tel.: +612-624-5541\\
           \email{lerman@math.umn.edu}
}

\date{Received: date / Accepted: date} 

\maketitle
\sloppy

\begin{abstract}
We present a new approach to rigid-body motion segmentation from
two views. We use a previously developed nonlinear embedding of two-view point correspondences into a 9-dimensional space and identify the different motions by segmenting lower-dimensional subspaces. In order to overcome  nonuniform distributions along the subspaces, whose dimensions are unknown, we suggest the
novel concept of global dimension and its minimization for clustering subspaces
with some theoretical motivation. We propose a fast projected gradient algorithm
for minimizing global dimension and thus segmenting motions from 2-views. We develop an outlier detection framework around the proposed method, and we present state-of-the-art results on outlier-free and outlier-corrupted two-view data for segmenting motion.

\end{abstract}

\noindent \textbf{Supp. webpage}: \url{http://math.umn.edu/~lerman/gdm}

\section{Introduction}
\label{intro}
A classic problem in computer vision is that of feature-based motion segmentation from two views. In this problem one has two images, taken at different times, of a 3D scene. The scene is assumed to consist of multiple, independently moving rigid bodies. The goal is to identify the different moving objects and estimate a motion model for each one of them. For this purpose, one automatically tracks the locations of visually-interesting ``features'' in the scene, which are visible in both views (e.g., via Lucas Kanade type algorithm~\cite{Lucas-Kanade-20years_2004}). Each feature is represented as a pair of 2-vectors, holding the image coordinates of the feature in the two different views; such a pair is referred to as a {\em point correspondence}. The mathematical problem of feature-based, two-view motion segmentation is to both segment the point correspondences according to the rigid objects to which they belong, and estimate a motion model for each object.

A basic strategy to solve the feature-based, two-view motion segmentation problem is to first cluster point correspondences and then estimate the single-body motions within clusters (well-known methods for single-body motion estimation are described in~\cite{Hartley2000,ma2004book}). This procedure was suggested in~\cite{FengPerona98}, while clustering point correspondences with $K$-means or spectral clustering, and in~\cite{Torr98geometricmotion}, while alternating between clustering and motion segmentation via an EM procedure. Both clustering strategies of \cite{FengPerona98} and \cite{Torr98geometricmotion} are based primarily on spatial separation between the clusters, however, different clusters in this setting may intersect each other (e.g., when motions share a symmetry).

Due to this problem, some algebraic methods have been developed for directly solving for the motion parameters, while eliminating the clustering of point correspondences~\cite{Vidal+2views06,rao_ijcv}.
Another solution is to segment feature trajectories by taking into account their geometric structures, which may be different than spatial separation (the feature trajectories in 2-views are the 4-dimensional vectors concatenating the 2 point correspondences of the same feature from 2 views).

Costeira and Kanade~\cite{Costeira98} showed that under the \emph{affine camera} model, feature trajectories in $n$-views (for $n \geq 2$ these are vectors of length $2 n$) within each rigid body lie on an affine subspace of dimension at most 3. This observation has given rise to several feature-based motion segmentation schemes for $n$-views, which are based on clustering subspaces; we refer to such clustering as Hybrid Linear Modeling (HLM).
Many algorithms have been suggested for solving the
HLM problem, for example, the $K$-flats (KF) algorithm or any of its
variants~\cite{Tipping99mixtures,Bradley00kplanes,Tseng00nearest,Ho03,MKF_workshop09},
methods based on direct matrix factorization~\cite{Boult91factorization-basedsegmentation,Costeira98,Kanatani01,Kanatani02},
Generalized Principal Component Analysis (GPCA)~\cite{Vidal05,Ma07,Ozay10},
Local Subspace Affinity (LSA)~\cite{Yan06LSA}, RANSAC (for HLM)~\cite{Yang06Robust},  Agglomerative Lossy
Compression (ALC)~\cite{Ma07Compression}, Spectral Curvature Clustering (SCC) \cite{spectral_applied}, Sparse Subspace Clustering (SSC)~\cite{ssc09,ssc_elhamifar13},
Local Best-Fit Flats (LBF and its spectral version SLBF)~\cite{LBF_cvpr10,LBF_journal12} and Low-rank Representation (LRR)~\cite{lrr_short,lrr_long}.
Some theoretical guarantees for particular HLM algorithms appear in~\cite{spectral_theory,higher-order,lp_recovery_part2_11,Soltanolkotabi2011,Soltanolkotabi+EC13,ALZ13}. Two recent reviews on HLM are by Vidal~\cite{SubspaceClustering_Vidal} and Aldroubi~\cite{Aldroubi_review_13}.

For the more general and realistic model of the \emph{perspective camera}, it can be shown that feature trajectories from two-views lie on quadratic surfaces of dimension at most 3 (in $\R^4$) (see \S 2).
Arias-Castro et al.~\cite{higher-order} suggested clustering the quadratic surfaces of point correspondences (in $\R^4$) using Higher Order Spectral Clustering (HOSC) for manifold clustering. They demonstrated competitive results on the outlier-free database of~\cite{rao_ijcv}, when assuming that the clusters are of dimension 2. However, their results are not competitive for incorporating dimension 3 and they did not provide any numerical evidence that the dimension of the surfaces was 2 and not 3.

A different approach for clustering these particular quadratic surfaces can be obtained by embedding point correspondences into ``quadratic coordinates'' and then clustering subspaces. More precisely, if a point correspondence $((x,y),(x',y'))$ is mapped into $(x,y,1) \otimes (x',y',1) \in \R^9$, where $\otimes$ denotes the Kronecker product, then these quadratic surfaces are mapped into linear subspaces of dimensions at most 8, which are determined by the fundamental matrices~\cite{Hartley2000,ma2004book,AtevKSCC} of the different motions. Chen et al.~\cite{AtevKSCC} have used this idea for clustering such quadratic mappings of point correspondences by the Spectral Curvature Clustering (SCC) algorithm~\cite{spectral_applied} (they showed that instead of performing the actual mapping, one can apply the kernel trick). They claimed that other HLM algorithms (at that time) did not work well for such embedded data.

The drawback of applying SCC to this quadratic mapping of point correspondences in $\R^9$ is that SCC does not work well with subspaces of mixed dimensions, and the subspace dimensions must be known a-priori. Unfortunately, the subspaces in this application have mixed and unknown dimensions (see \S\ref{sec:Formulating 2-View Segmentation}).
What makes SCC successful for this application is the fact that it takes into account some global information of the subspaces (i.e., for $d$-dimensional subspaces it uses affinities based on arbitrary $d+2$ points, and in particular, far-away points). This helps SCC deal with nonuniform sampling along subspaces with local structure very different than the global one (see \S\ref{sec:Formulating 2-View Segmentation}). On the other hand, local methods (e.g., \cite{LBF_journal12,mapa}) often do not work well in this setting.

The purpose of this paper is to develop an HLM algorithm that can successfully cluster the quadratically-embedded point correspondences in $\R^9$. In particular, it exploits the global structure of the underlying subspaces, i.e., their ``dimensions''. We remark that earlier works \cite{Barbará00usingthe,Gionis:2005:DIC:1081870.1081880,Haro06,Haro08TPMM} used dimension estimators to segment data clusters according to their intrinsic dimension (\cite{Barbará00usingthe} and \cite{Gionis:2005:DIC:1081870.1081880} used box counting estimators of some fractal dimensions and \cite{Haro06,Haro08TPMM} used the statistical estimator of~\cite{Levina05}). However, their methods do not distinguish well subspaces of the same dimension. Here on the other hand, we aim to minimize ``dimensions'' within tentative clusters, instead of estimating them. Thus, we take into account the effect of tentative clusters on these ``dimensions'' and try to optimize accordingly the appropriate choice of clusters.

For this purpose, we propose a class of empirical dimension estimators, and a corresponding notion of global dimension for a mixture of subspaces (a function of the estimated dimensions of its constituent parts).
We propose the global dimension minimization (GDM) algorithm, which is a fast projected gradient method aiming to minimize the global dimension among all data partitions. We also build an outlier detection framework into this development to allow for corrupted data sets. We demonstrate state-of-the-art results for two-view motion segmentation (via quadratic embedding), both in the outlier-free and outlier-corrupted cases.
We even show that these results are competitive with the state-of-the-art results for multiple-views, i.e., using all frames of a video sequence (obtained under the affine camera model). To motivate the use of global dimension, we prove that for special settings and choice of parameters, the global dimension is minimized by the correct partition of the data (representing the underlying subspaces). We then discuss what to do in more general settings.

The paper is organized as follows:
\S\ref{sec:Formulating 2-View Segmentation} briefly explains how the problem of 2-view motion segmentation can be formulated as a problem in HLM; \S\ref{sec:On Global Dimension} introduces global dimension and explains why its minimization can solve the HLM problem under some conditions; \S\ref{sec:A Variational Algorithm} develops a fast projected gradient method for minimizing global dimension; \S\ref{sec:Outlier Detection} develops an outlier detection/rejection framework for global dimension minimization; \S\ref{sec:Results} demonstrates numerical results on real-world 2-view data sets for both outlier-removed and outlier-corrupted data; finally, \S\ref{sec:Conclusion} concludes this work. The appendix contains proofs of the key results in the paper.

\section{Formulating 2-View Motion Segmentation as a Problem in HLM}
\label{sec:Formulating 2-View Segmentation}

One way of formulating the motion segmentation problem in terms of HLM is by exploiting the \emph{Affine Motion Subspace}. Costeira and Kanade~\cite{Costeira98} demonstrated that when a set of features all come from a single rigid body, then under the assumptions of the affine camera model, the corresponding feature trajectories lie in an affine subspace of dimension 3 or less. One can use this fact to partition the set of features by clustering their trajectories into subspaces. This is a popular formulation of the segmentation problem, even when dealing with only two views of a scene.

The formulation involving the affine motion subspace has the advantage that the feature trajectories tend to be nicely distributed in their respective subspaces, and the different subspaces all have nearly the same dimensions. This formulation has the drawback that it requires an affine camera model. The affine camera assumption breaks down when viewing objects close to the camera, or when looking at objects at significantly different ranges. The consequence of this is that the trajectories from a rigid body do not lie within a subspace, but rather in a manifold which is only locally approximated by a subspace of dimension at most 3.

When dealing with 2-view segmentation, there is another approach, based on a more general camera model, which avoids this problem of distortion. This approach assumes a perspective camera, and relies on the fundamental matrix~\cite{Hartley2000} for a rigid body.

Indeed, if $\mathbf{F}=(F_{i,j})_{i,j=1}^3$ is the fundamental matrix for a rigid body, and $\mathbf{x}_h=(x,y,1)^T$ and $\mathbf{x}'_h=(x',y',1)^T$ together form a point correspondence (in standard homogenous coordinates) from that body, then
\begin{equation}
\notag {\mathbf{x}'_h}^T \vec{F} \mathbf{x}_h=0,
\end{equation}
which is algebraically equivalent to
\begin{equation}
\label{eq:Fv0}
\textrm{vec}(\vec{F}) \cdot \bv = 0,
\end{equation}
where
\begin{equation}
\notag \textrm{vec}(\vec{F}) = (F_{11}, F_{12}, F_{13}, F_{21}, F_{22}, F_{23}, F_{31}, F_{32}, F_{33})^T
\end{equation}
and
\begin{equation}
\notag \bv = (xx',  x'y,  x',  xy',  yy',  y',  x,  y, 1)^T = (x,  y,  1)^T \otimes (x',  y',  1)^T.
\end{equation}

We refer to the vector $\bv$ as the \emph{nonlinear} or \emph{Kronecker} embedding of a point correspondence (recall that $ \otimes$ is the Kronecker product). The vectors obtained through this nonlinear embedding for feature points on the same rigid object lie in a linear subspace of $\R^9$ of dimension at most 8. Indeed, \eqref{eq:Fv0} says that there is a vector $\textrm{vec}(\vec{F}) \in \R^9$
orthogonal to all of the feature trajectories in this set (it also shows that the linear embedding $(x,y,x',y')$ lies on a 3-dimensional quadratic manifold).
However, the subspace dimension can decrease due to two different reasons. First of all, if there are very few points (per motion), then they may span a lower-dimensional subspace. The second cause is degeneracy in the 3D configuration of the features. If all world points and both camera centers live on a ruled quadratic surface\footnote{A surface $S$ is ruled if through every point of $S$ there exists a straight line that lies on $S$.}, then their corresponding subspace has dimension 7 or less. In particular, if all world points (but not necessarily the camera centres) are coplanar, the corresponding subspace will have dimension no larger than 6 (see~\cite[pg. 296]{Hartley2000}). Therefore, to make use of this embedding, the hybrid-linear modeling algorithm being employed must be tolerant of subspaces of mixed dimension.

\begin{figure}
\captionsetup{margin=10pt,labelfont=bf}
\centering
\begin{minipage}[b]{0.27\linewidth}
\centering
\includegraphics[width=\linewidth, clip=false]{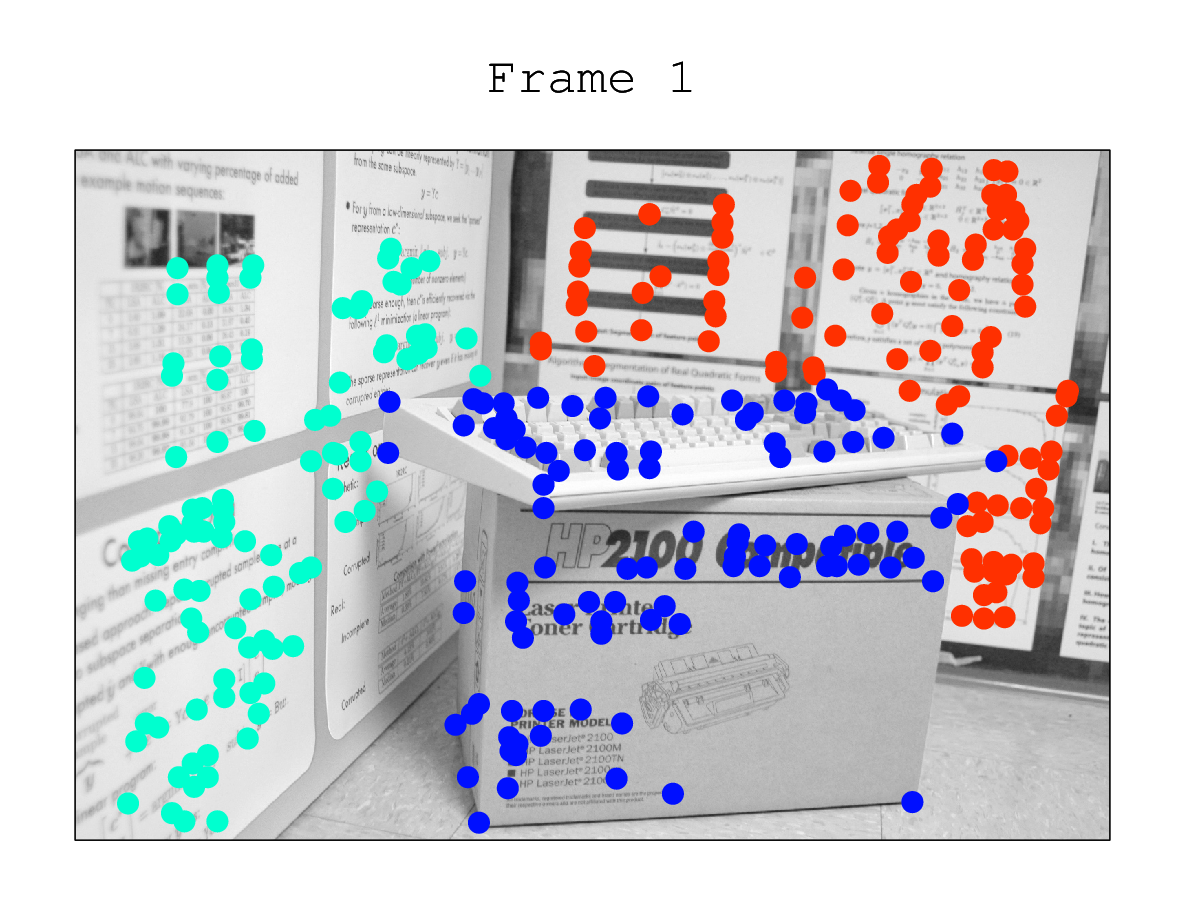} \\
\includegraphics[width=\linewidth, clip=false]{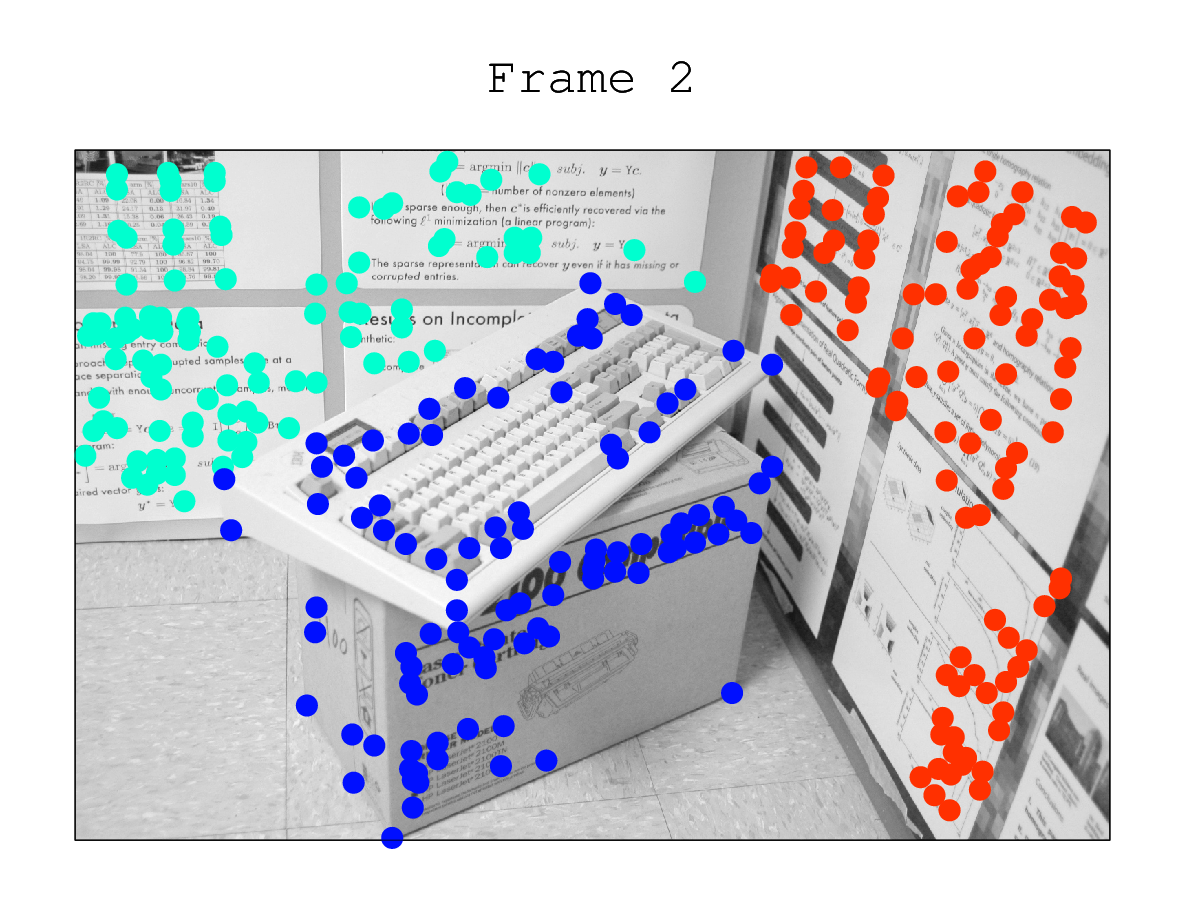}
\label{fig:figure1}
\end{minipage}
\hspace{1 mm}
\includegraphics[width=.6\linewidth, clip=true, trim=40mm 34mm 40mm 35mm]{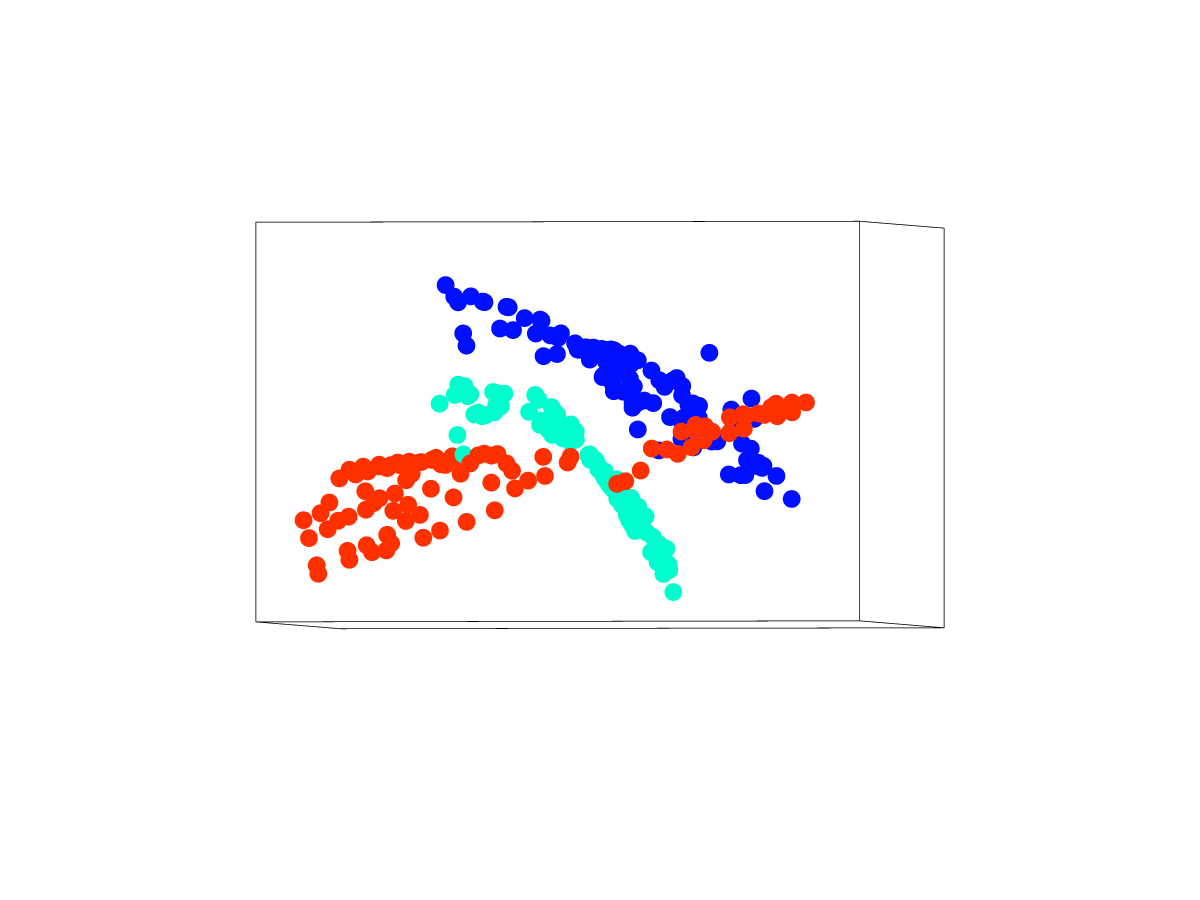}
\caption{\label{fig:Cartoon}Two views of a 3D scene with features overlayed (left), and the nonlinearly embedded point correspondences in $\R^9$, projected onto the 3-dimensional subspace spanned by their 3rd, 4th and 5th principal components (right). (Color figure online)}
\end{figure}

Since the perspective camera assumption is accurate in a much broader range of situations than the affine camera model, subspaces are more apparent with the nonlinear embedding than with the linear embedding. However, the nonlinear embedding distorts the original sampling and results in lower-dimensional structures (of dimension at most 3) within the higher dimensional subspaces (of typical dimensions 6, 7 or 8), which is a serious obstacle for many HLM algorithms, especially ones using local spatial information.

\section{Global Dimension}
\label{sec:On Global Dimension}

From here on, we will be considering 2-view motion segmentation under the perspective camera model, i.e. using the Kronecker embedding. We will present a global HLM method, which is well-suited for handling the data which results from this embedding. We begin our development by providing some intuitive motivation for our approach.

Imagine that we have access to an oracle, who for any set of vectors in $\R^D$, can provide for us a good, robust estimate of the dimensionality of the set\footnote{In a noiseless case this would return the dimension of the linear span of the set of vectors}. Now, suppose we have a set of vectors which are sampled from a hybrid-linear distribution. Consider a general partition of the data set, and define the ``vector of set dimensions'' for that partition to be the vector of oracle-provided approximate dimensions of each respective set in the partition. Our inspiration is the observation that for most partitions one may happen upon, each set in the partition will typically contain points from many of the underlying subspaces. The associated vector of set dimensions will contain relatively large numbers, and the $p$-norm of this vector will be large. The $p$-norm of the vector of set dimensions will be referred to as the \emph{global dimension} of the partition. The best way to make the global dimension small, it would seem, is to try and decrease all of the elements of the vector of set dimensions by grouping together vectors that come from common subspaces. This notion will be made precise and we will show, in fact, that under certain conditions, the \emph{natural partition} of the data set (the one where point assignment agrees with subspace affiliation) is a global minimizer of the global dimension function.

	Our approach to HLM will be to find the partition of a data set that yields the lowest possible global dimension. In this section, we develop the global dimension objective function in two parts. In \S\ref{sec:Empirical Dimension} we suggest a new class of dimension estimators that can perform the role of the oracle in our discussion above. In \S\ref{sec:Defining Global Dimension} we define \emph{global dimension} and explain why we expect its minimizer to reveal the clusters corresponding to the underlying subspaces. A fast algorithm for this minimization will be later described in \S\ref{sec:A Variational Algorithm}.

\subsection{On Empirical Dimension}
\label{sec:Empirical Dimension}
We present here a class of dimension estimators depending on a parameter $\epsilon \in (0,1]$. For $\bu=(u_1,\ldots,u_k)$ and any $p>0$, we use the notation $\|\bu\|_p$ to mean $(u_1^p+ \ldots + u_k^p)^{1/p}$ (even for $p=\epsilon<1$, where $\| \cdot \|_p$ is not a norm).
For a given set of $N$ vectors in $\R^D$, $\{\bv_i\}_{i=1}^{N}$, we denote by $\bsigma = (\sigma_1 \; \sigma_2 \; \ldots \; \sigma_{N \land D})^T$  the vector of singular values of the $D \times N$ data matrix $\A$ (the matrix whose columns are the data vectors).
\begin{sloppypar}
For $\epsilon \in (0,1]$ the \emph{empirical dimension}, denoted by $\hat{d}_\epsilon(\bv_1,\bv_2,...,\bv_N)$ (or simply $\hat{d}_\epsilon$) is defined by
\end{sloppypar}
\begin{equation}
\hat{d}_\epsilon(\bv_1,\bv_2,...,\bv_N) := {{\Vert \bsigma \Vert_\epsilon} \over {\Vert \bsigma \Vert_{\left({{\epsilon} \over {1-\epsilon}}\right)}}}.
\end{equation}
When $\epsilon=1$, this is sometimes called the ``effective rank''\footnote{``Effective rank'' is sometimes defined differently. See \cite{Roy_theeffective}.} of the data matrix~\cite{vershynin_book}.

The following theorem explains why $\hat{d}_\epsilon$ is a good estimator for dimension. Put simply, it says two things. First, if we rotate and/or uniformly scale our set of vectors by a non-zero amount, then the empirical dimension of the set does not change. Second, in the absence of noise, empirical dimension never exceeds true dimension, but it approaches true dimension in the limit (as the number of measurements goes to infinity) for spherically symmetric distributions. From now on we refer to $d$-dimensional subspaces as $d$-subspaces.
\begin{theorem}
\label{thm:Empirical Dimension Properties}
For $\epsilon \in (0,1]$, $\hat{d}_\epsilon$ possesses the following properties:
\begin{enumerate}
\item $\hat{d}_\epsilon$ is invariant under dilations (i.e., scaling).
\item $\hat{d}_\epsilon$ is invariant under orthogonal transformations.
\item \label{prop_dim_bound}
If $\{\bv_i\}_{i=1}^{N}$ are contained in a $d$-subspace of $\R^D$, then $\hat{d}_\epsilon \leq d$.
\item \label{prop_dim_equiv}
If $\{\bv_i\}_{i=1}^{N}$ are i.i.d.~samples from a sub-Gaussian probability measure, which is spherically symmetric within a
$d$-subspace\footnote{A measure is spherically symmetric within a $d$-subspace if it is supported on this subspace and invariant to rotations within this subspace.} and non-degenerate\footnote{A measure is non-degenerate on a subspace if it does not concentrate mass on any proper subspace. In our setting the measure is also assumed to be spherically symmetric, and this assumption is equivalent to assuming the measure does not concentrate at the origin.}, then $\lim_{N \rightarrow \infty} \hat{d}_{\epsilon}(\bv_1,\ldots,\bv_N) = d$ with probability 1.
\end{enumerate}
\end{theorem}

To gain some intuition into the definition of empirical dimension, consider taking a large set of samples from a spherically symmetric distribution supported by a $d$-subspace. Call the covariance matrix for this distribution $\mathbf{Q}$. As the number of samples becomes large, the empirical covariance matrix approaches $\mathbf{Q}$, which has the first $d$ elements on the main diagonal all equal (call the value $\alpha^2$), and 0's everywhere else. The empirical dimension of the set of vectors involves the singular values of the data matrix, which are approaching the square roots of the eigenvalues of $\mathbf{Q}$. Hence, as the number of samples increases, we get:

{\small
\begin{equation}
\hat{d}_{\epsilon}(\bv_1,\ldots,\bv_N) \rightarrow {{\Vert (\alpha, \alpha, ..., \alpha, 0, ..., 0) \Vert_\epsilon} \over {\Vert (\alpha, \alpha, ..., \alpha, 0, ..., 0) \Vert_{\left({{\epsilon} \over {1-\epsilon}}\right)}}} = {{d^{1/\epsilon} \alpha} \over {d^{(1-\epsilon)/\epsilon} \alpha}} = d.
\end{equation}
}
Thus, for any value of $\epsilon$ in $(0,1]$, the empirical dimension approaches the true dimension of the set as the number of measurements increases.

If we look at a distribution that is not spherically symmetric, but still supported by a $d$-subspace, then empirical dimension tends to under-estimate the true dimension of the distribution, even as the number of samples approaches infinity. This is actually desirable behavior. If we take a spherically symmetric distribution in a $d$-subspace and imagine the process of collapsing it in one direction until it lies in a $(d-1)$-subspace, then true dimension behaves discontinuously. The true dimension of a large set of samples will equal $d$ until the collapsing is complete; at that point the dimension will instantly drop to $d-1$. Empirical dimension smoothly drops from $d$ to $d-1$ during this collapsing process. It is in this setting that we see the necessity of the parameter $\epsilon$. This parameter controls how quickly the empirical dimension drops from $d$ to $d-1$ in this process. More generally, a low value of $\epsilon$ results in a ``strict'' dimension estimator (meaning that it will not under-estimate dimension easily, even when distributions are asymmetric). When $\epsilon$ is large (approaching 1), empirical dimension is a lenient dimension estimator. It is much more tolerant of noise, but it may consequently under-estimate the dimension of highly asymmetric distributions. The trade-off is that when dealing with noisy data or distributions only approximately supported by linear subspaces, a stricter estimator can mistakenly interpret noise or distortion as energy in new directions, thereby causing an over-estimate of dimension. Numerical experiments (e.g., Fig. \ref{fig:EmpiricalDim}) have shown that values of $\epsilon$ between 0.3 and 0.7 seem to provide reasonable estimators, which tend to agree with our intuitive notion of dimension.
\begin{figure}[h]
\captionsetup{margin=10pt,labelfont=bf}
\centering
\includegraphics[width=\linewidth, clip=true, trim=15mm 20mm 15mm 20mm]{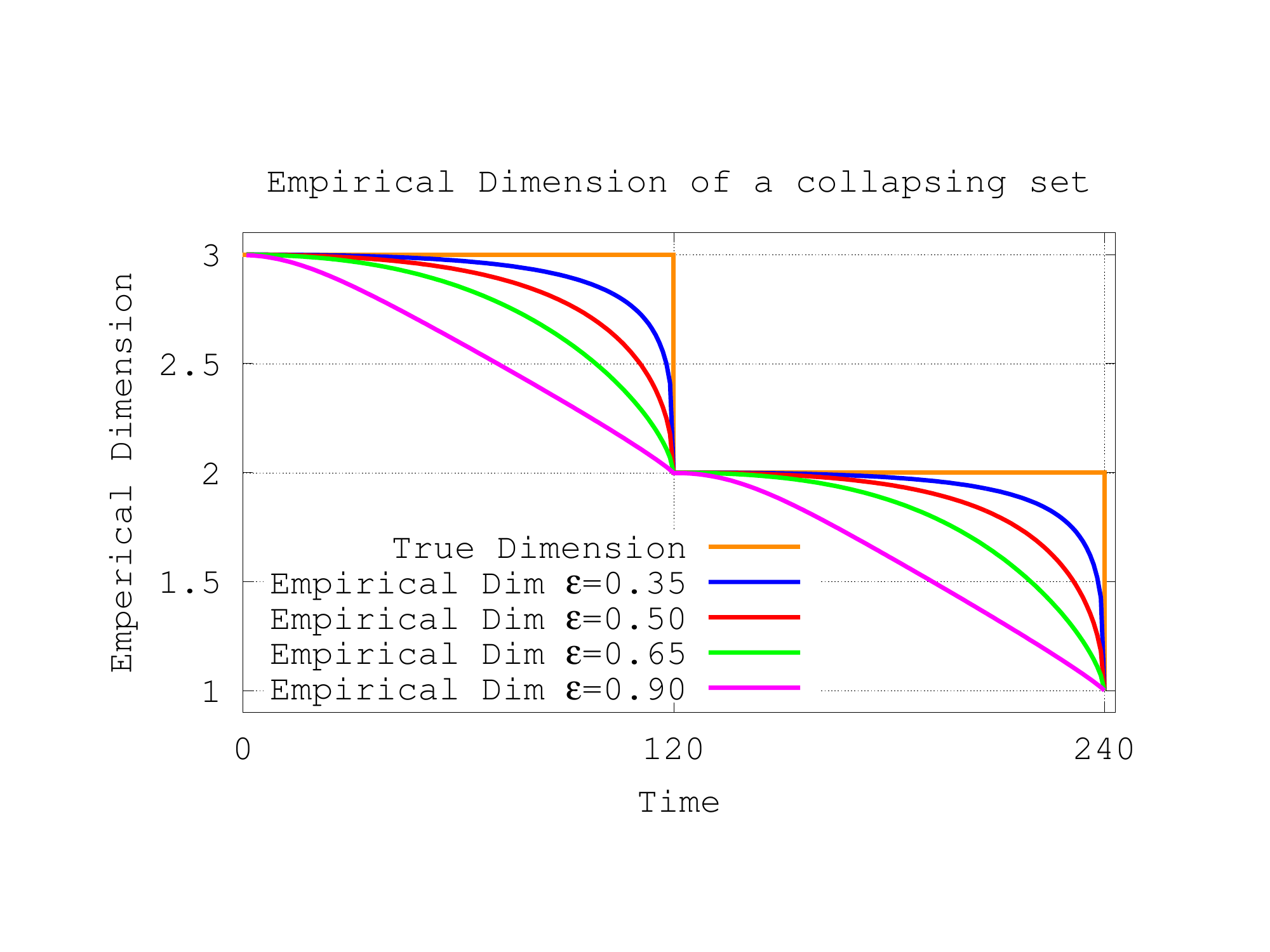}
\caption{\label{fig:EmpiricalDim}The experiment mentioned above is illustrated. A normally-distributed point cloud is created in $\R^3$ and is slowly collapsed into a plane and then a line. One can see that if $\epsilon$ is close to $0$, empirical dimension more closely tracks true dimension, resulting in a strict dimension estimator. If $\epsilon$ is close to $1$, empirical dimension changes more smoothly, resulting in a lenient estimator. (Color figure online)}
\end{figure}

In our application, since the perspective camera model is reasonably accurate (as opposed to the affine model), the nonlinear embedding of point correspondences results in subspaces with rather negligible distortion. We can thus afford a low value of $\epsilon$. In fact, this is needed because the data vectors are frequently distributed in very non-isotropic ways with this embedding. Thus, to avoid underestimating dimension, we choose $\epsilon = 0.35$, which lies just slightly above the lowest value we confirmed for $\epsilon$ (0.3). Notice that this value is not ``tuned'' to individual data sets, but is chosen based on the properties of the application as a whole and the nature of the embedding.

\subsection{On Global Dimension}
\label{sec:Defining Global Dimension}
Assume we are provided a data set $X$ in $\R^D$ (in our application $D=9$ with the nonlinear embedding) and a partition of it $\Pi=\left( \Pi_1,\Pi_2,...,\Pi_k\right)$ for some $k \in \N$ (i.e., $\{\Pi_i\}_{i=1}^k$ are disjoint subsets of $X$ whose union is $X$).
We also assume that $X$ lies on a union of $K$ subspaces and denote the ``correct'' (or natural) partition of the data (where each subset contains only points from a single underlying subspace) by $\Pi_{Nat}$.
For a fixed $\epsilon \in (0,1]$, $\{\hat{d}_{\epsilon, i}\}_{i=1}^k$ are the empirical dimensions of the sets $\{\Pi_i\}_{i=1}^k$. We seek to minimize a function based on these dimensions to recover $\Pi_{Nat}$. To this end we define \emph{global dimension} (GD). When thinking of this function, we take the set of data vectors to be fixed and given, and we view GD as a function of partitions, $\Pi$, of the set of data vectors. For a fixed $p \in (0,\infty)$ (we discuss the meaning of $p$ later) we define GD as follows:

\begin{equation}
\textrm{GD}(\Pi) = \Vert (\hat{d}_{\epsilon,1}, \hat{d}_{\epsilon,2}, ..., \hat{d}_{\epsilon,K})^T \Vert_p =  \left(
\sum_{i=1}^K  {\hat{d}_{\epsilon, i}}^p \right)^{1/p}.
\end{equation}

Our strategy for recovering $\Pi_{Nat}$ will be to try and find the partition of the data set that minimizes $\textrm{GD}(\Pi)$. Intuitively, by trying to minimize the $p$-norm of the vector of set dimensions, we are looking for a partition where all of the set dimensions are small. Imagine trying to minimize this objective function by hand, and starting with a partition close to, but not equal to $\Pi_{Nat}$. If there is a point assigned to the wrong cluster, then removing it from the set it is currently assigned to should result in a significant drop in the dimension of that particular set. Re-assigning that point to the correct set, on the other hand, will have little impact on the dimension of the target set because the point will lie approximately in the span of other points already in the set.

Thus, such a change would cause a significant drop in one of the set dimensions, without disturbing the other sets, and the global dimension will decrease. This would suggest that amongst partitions that are close to it, $\Pi_{Nat}$ yields the lowest global dimension. Additionally, if one considers a ``random'', or usual partition, then each set in that partition will tend to contain vectors from many different subspaces. Each set will have a large dimension, and the global dimension will exceed that of $\Pi_{Nat}$. This would suggest that minimizing global dimension may be a reasonable objective if we want to recover $\Pi_{Nat}$.

Unfortunately, there can exist certain special partitions of a data set that result in low global dimension (in some cases even lower than that of $\Pi_{Nat}$). For example, let us choose $p=1$, so that the global dimension of a partition is simply the sum of the dimensions of its constituent parts. Now consider 3 lines in the plane, and a data set consisting of many points sampled from each line. In this case $\Pi_{Nat}$ will consist of three sets. Each set will contain only points from a single line. The dimension of each set in $\Pi_{Nat}$ is 1. Hence, $\textrm{GD}(\Pi_{Nat}) = 3$. On the other hand, if we consider the ``degenerate'' partition, that simply puts all points in a single set, then since we are in $\R^2$, the dimension of that set, and hence the global dimension of the partition, is 2.

The above example is actually rather special. Consider the same data set, but set $p$ to a large value instead of 1. When $p$ is large, the global dimension approximately returns the largest value from $\{\hat{d}_i\}_{i=1}^K$. Now consider minimizing this quantity, subject to the constraint that the partition contains no more than 3 sets. Minimizing global dimension in this setting penalizes partitions consisting of fewer, higher-dimensional sets instead of multiple, more balanced sets. Specifically, the global dimension of the degenerate partition is again approximately 2, while the global dimension of $\Pi_{Nat}$ is approximately 1, since that is the maximum dimension of its constituent sets. In fact, as shown in the next theorem, using large $p$ effectively resolves the issue of special partitions yielding lower global dimension than $\Pi_{Nat}$.

We will consider the setting where we have $K$ distinct linear subspaces of $\R^D$, each of dimension $d < D$. Call these subspaces $\{ L_k \}_{k=1}^K$. Assume we have a collection of non-degenerate measures $\{ \mu_k \}_{k=1}^K$ supported by these subspaces (so that $\mu_k$ is supported by $L_k$, $k=1,2,...,K$). Let $\{\bv_n\}_{n=1}^N$ be a set consisting of $N_k$ i.i.d. points from each $\mu_k$ (so $N=N_1+N_2+...+N_K$). We require that $N_k > d$ for each $k$ so that each subspace is adequately represented in the data set. Let $GD_{True}$ be global dimension for a fixed parameter $p$, defined using true dimension as the ``dimension estimator'' for a set. That is, $GD_{True}(\Pi)=\Vert ( d_{True}(\Pi_1), ..., d_{True}(\Pi_K) ) \Vert_p$ where the sets $\Pi_k$ are the constituent sets of the partition $\Pi$ and $d_{True}(\bullet)$ returns the true dimension of its parameter set. Then, we get the following result:

\begin{theorem}
\label{THM:Global Dimension Theorem}
Let $\{ L_k \}_{k=1}^K$, $\{ \mu_k \}_{k=1}^K$, $\{\bv_n\}_{n=1}^N$ satisfy the conditions above. If
\begin{equation}
\label{eq:p_lower_bound}
p > ln(K)/(ln(d+1) - ln(d)),
\end{equation}
then amongst all partitions of $\{\bv_n\}_{n=1}^N$ into $K$ or fewer sets, the natural partition is almost surely (w.r.t $\{ \mu_k \}_{k=1}^K$) the unique minimizer of $GD_{True}$.
\end{theorem}

The weakness of the above theorem is that it requires all of the intrinsic subspaces to have the same dimension. In practice, the global dimension objective function appears to be rather robust to subspaces with mixed dimensions. If there is a large difference in dimension between two subspaces in a dataset, then the minimum of global dimension tends to be very near $\Pi_{Nat}$, the only difference being that a few points from the higher-dimensional set are re-assigned to lower-dimensional sets to balance out the set dimensions.

Theorem \ref{THM:Global Dimension Theorem} gives us a quantitative way of selecting an appropriate value of $p$ for our application. Specifically, if we identify the largest number of clusters we will need to address ($K$) and an upper bound for $d$, then the right-hand side (RHS) of~\eqref{eq:p_lower_bound} gives us a lower bound on the value of $p$. As long as $p$ is larger than this bound, then Theorem \ref{THM:Global Dimension Theorem} ensures that the natural partition (uniquely) minimizes global dimension. Table \ref{Table:Min p} exemplifies the values of the RHS of~\eqref{eq:p_lower_bound} for different values of $d$ and $K$. We do not want to choose $p$ extravagantly large because of the potential for numerical issues when taking large powers. In our setting, we want to be able to handle up to $4$ sets and we will use one less than the ambient dimension as an upper bound for the intrinsic dimension ($d=8$). According to Theorem \ref{THM:Global Dimension Theorem} we should select $p \ge 11.77$. In all of our experiments we set $p=15$ to give us a safety margin. We did some tests on one of our motion segmentation databases (outlier-free RAS) to see how sensitive the minimizer of global dimension is to $p$ in practice. We found that values as low as $p=10$ result in nearly identical performance to $p=15$, and we don't start to see significant degradation in results in the other direction until $p=25$.
\begin{table}[h]
\caption{Values of the RHS of~\eqref{eq:p_lower_bound} for various values of $K$ and $d$. Theorem \ref{THM:Global Dimension Theorem} ensures that $\Pi_{Nat}$ is the unique minimizer of $G_p$ when $p$ is larger than these values}
\centering
\tabcolsep=0.33cm
\begin{tabular}{ccc|c|c|c|c|}
\cline{3-7}
 & & \multicolumn{5}{|c|}{$d$}\\ \cline{3-7}
 & & \multicolumn{1}{|c|}{8} & 7 & 6 & 5 & 4\\ \hline
\multicolumn{1}{|c|}{\multirow{3}{*}{$K$}} & \multicolumn{1}{c|}{2} & 5.89 & 5.19 & 4.50 & 3.80 & 3.11 \\ \cline{2-7}
\multicolumn{1}{|c|}{} & \multicolumn{1}{c|}{3} & 9.33 & 8.23 & 7.13 & 6.03 & 4.92 \\ \cline{2-7}
\multicolumn{1}{|c|}{} & \multicolumn{1}{c|}{4} & 11.77 & 10.38 & 8.99 & 7.60 & 6.21 \\ \hline
\end{tabular}
\label{Table:Min p}
\end{table}

\section{A Fast Algorithm for Minimizing Global Dimension}
\label{sec:A Variational Algorithm}
Global dimension is defined on the set of partitions of a data set. With a discrete domain, finding ways of quickly minimizing the objective function is non-trivial. In this section we briefly introduce a method, which we will call \emph{Global Dimension Minimization} (GDM) for doing exactly this.

GDM is based on the gradient projection method~\cite[\S2.3]{bertsekas1995nonlinear}. In order to apply a gradient-based method, we need to re-formulate the problem so that we have a smooth objective function over a convex domain. To do this we employ the notion of \emph{fuzzy assignment}. Rather than trying to assign each data point a label, identifying it with a single cluster, we allow each point to be associated with every cluster simultaneously, in varying amounts. Specifically, we assign each data point $\bv_j$ a probability vector where the $i$'th coordinate holds the strength of $\bv_j$'s affiliation with cluster $i$. Assuming we have a data set of $N$ points in $\R^D$, and we seek $K$ clusters, we need $N$ probability vectors of length $K$ to encode the \emph{soft partition} of the data. This membership information will be stored in a \emph{membership matrix}, $\M$, where each column is a probability vector. Element $(i,j)$ of the matrix $\M$ holds the strength of $\bv_j$'s affiliation with cluster $i$.

The next step is to extend the definition of global dimension so that it is defined on soft partitions in a meaningful way. In its original formulation, to evaluate the global dimension of a partition, we would break up the data set into parts, based on the partition, and estimate the dimension of each part using empirical dimension. To extend this to soft partitions, we estimate the dimension of the $k$'th set in a partition by scaling each data point by its respective affiliation strength to set $k$ ($\bv_n$ is multiplied by $\M_{(k,n)}$). We then use empirical dimension to estimate the dimension of the scaled set. In essence, each point is now included in each dimension estimate. However, if a point is scaled so that it lays near the origin when considering a given set, it has little impact on the estimated dimension of that set. In fact, if we look at the global dimension of a soft partition that assigns each data point entirely to a single set ($\M$ has only 1's and 0's in it), then the global dimension of that soft partition, using our new definition, agrees with the global dimension of the corresponding ``hard partition'', using our original definition. Thus, this change is a reasonable extension of the original definition to soft partitions. Our extended definition of global dimension is:
\begin{align}
\label{Def: Fuzzy Global Dimension}
&GD = \Vert (\hat{d}_\epsilon^1, \hat{d}_\epsilon^2, ..., \hat{d}_\epsilon^K) \Vert_p \\
&\notag \text{where } \ \hat{d}_\epsilon^k = \hat{d}_\epsilon(\M_{(k,1)}\bv_1,\M_{(k,2)}\bv_2,...,\M_{(k,N)}\bv_N).
\end{align}
With this modified formulation, global dimension is an almost-everywhere differentiable function defined over the Cartesian product of $N$ $K$-dimensional probability simplexes. One can check that this is a convex domain (the product of convex sets is convex). A natural approach to minimizing a problem of this sort is the gradient projection method~\cite[\S2.3]{bertsekas1995nonlinear}. In this method, we begin at some initial state, compute the gradient of the objective function, take a step in the direction opposite the gradient, and then project our new state back into the domain of optimization. This is repeated until our state converges.

The gradient of global dimension can be computed, but we need some notation first. For $i=1, \ldots, K$, we denote by $\A_k$ the $D$-by-$N$ matrix whose $j$'th column equals $\M_{(k,j)} \bv_j$ for $j=1,2,...,N$ (i.e., $\A_k$ is the data matrix scaled according to weights for cluster $k$). Let $\A_k = \U_k \bSigma_k (\V_k)^T$ be the thin SVD of $\A_k$, and $\bsigma_k$ denote the vector of elements from the diagonal of $\bSigma_k$. Let $\delta = \epsilon / (1 - \epsilon)$. Define
{\small
\begin{equation}
\label{eq:prelim_DG}
\D_k = \left( {{\Vert \bsigma_k \Vert_{\epsilon}^{1-\epsilon} \Vert \bsigma_k \Vert_{\delta}} \over {\Vert \bsigma_k \Vert_{\delta}^2}} \right) \cdot (\bSigma_k)^{\epsilon - 1}
-
\left( {{\Vert \bsigma_k \Vert_{\epsilon} \Vert \bsigma_k \Vert_{\delta}^{1-\delta}} \over {\Vert \bsigma_k \Vert_{\delta}^2}} \right) \cdot
(\bSigma_k)^{\delta - 1}.
\end{equation}
}
\begin{theorem} \label{thm: Derivative of GD} The derivative of global dimension w.r.t.~an arbitrary element of the membership matrix $\M$ is given by:
{\small
\begin{equation}
\label{thm:DG}
{{\partial GD} \over {\partial \M_{(k,n)}}} =  {\V_k}_{{(n,:)}} \left( (\hat{d}_\epsilon^k)^{p-1} \,  \Vert (\hat{d}_\epsilon^1, \hat{d}_\epsilon^2, ..., \hat{d}_\epsilon^K) \,  \Vert_p^{1-p}  \, \D_k \, (\U_k)^T \right) \A_{(:,n)}.
\end{equation}
}
\end{theorem}
A proof of Theorem \ref{thm: Derivative of GD} is included in the appendix.
This theorem allows us to evaluate the gradient vector of global dimension. As was mentioned before, in an iteration of the gradient projection method we take a step in the direction opposite the gradient. Computing a good step size is frequently a challenging task, but here we are fortunate. Our domain has a meaningful natural scale, since it is formed as a product of probability simplexes. Intuitively, our step size should be large enough to move us across the entire space in a reasonable number of steps, but small enough that any individual membership vector can move only a fraction of the way across its own simplex in one step. In practice, we scale each step so that the membership vectors most affected by the step move a distance of .3 on average. This seems to work well in general.

Finally, one can check that projecting onto the domain of optimization can be accomplished by individually projecting each column of $\M$ onto the standard $K$-dimensional probability simplex.

We have outlined a projected gradient descent method for minimizing global dimension. The above method forms the core of the GDM algorithm. However, since the global dimension function is non-convex (and hence may contain multiple local minimums), it is important to achieve reasonably good initialization. Our initialization strategy is inspired by ALC~\cite{Ma07Compression}. We start with a ``trivial'' partition where each point is in its own set, and we randomly select many pairs of sets in the partition. For each pair, we hypothetically merge the two sets and measure the resulting global dimension. We select the pair that results in the lowest global dimension when merged and we effect that merge. We then repeat the process iteratively until we have the desired number of sets in our partition (in each step the number of sets in the partition decreases by $1$). After initialization, the projected gradient descent algorithm is run until convergence (or for a fixed, but large number of iterations). Thresholding is performed to recover a ``hard partition'' from our soft partition (point $j$ is assigned to cluster $i$ if $\M_{(i,j)}$ is the largest element from column $j$ of $\M$). After this is done, we perform a final genetic stage to clean up small errors which may have occurred in any of the previous stages. This is done by taking each point and hypothetically re-assigning it to each different cluster (while all other point assignments are kept fixed) and retaining the assignment that results in the lowest global dimension. This is repeated a few times or until no single-point re-assignments reduce global dimension. This primarily helps with placing points that lie near the intersections of different subspaces (their fuzzy assignments may associate them almost equally to $2$ different subspaces, making them difficult to place). Finally, we run this entire process several times and return the best partition of all runs (as measured by global dimension).

\begin{algorithm}[ht]
\label{Algorithm}
\caption{GDM Algorithm for HLM}
\begin{algorithmic}
\REQUIRE $X=\{\bx_1,\bx_2,\cdots,\bx_N\} \subseteq \R^D$: data, $K$: number of clusters, $p$: global dimension parameter, $\epsilon$: empirical dimension parameter, $n_1$, $n_2$, $n_3$: number of iterations (default: $n_1=n_3=10$, $n_2=30$)
\ENSURE  A partition, $\Pi$, of $X$ into $K$ disjoint clusters
\FOR {$i=1:n_1$}
     \item $\bullet$ $\Pi:=$ Partition of $X$ where each point is in its own set.
     \WHILE {number of sets in $\Pi$ greater than $K$}
     	\item $\bullet$ Randomly choose several pairs of sets.
     	\item $\bullet$ For each pair, measure the effect on global dimension if the pair is merged.
     	\item $\bullet$ Merge the pair of sets which results in the lowest global dimension.
     \ENDWHILE
     \item $\bullet$ Convert $\Pi$ to a soft partition, encoded in membership matrix $M$.
     \FOR {$j=1:n_2$}
     	\item $\bullet$ Compute gradient of global dimension, $\nabla GD$.
     	\item $\bullet$ $\rho$ = average magnitude of largest 10\% of columns of $\nabla GD$.
     	\item $\bullet$ Take a step in direction $-1 * \nabla GD$ of length $.3/\rho$.
     	\item $\bullet$ Project each column of $M$ onto the standard $k$-dimensional probability simplex.
     \ENDFOR
     \item $\bullet$ Convert $M$ back to a ``hard partition'', $\Pi$, by thresholding.
     \FOR {$j=1:n_3$}
     	\FOR {$n=1:N$}
     		\item $\bullet$ Check if re-assigning point $n$ to some other cluster decreases global dimension.
     		\item $\bullet$ If so, re-assign point $n$ to that cluster.
     	\ENDFOR
     \ENDFOR
\ENDFOR
\item $\bullet$ return partition from all runs with lowest global dimension.
\end{algorithmic}
\end{algorithm}

\subsection{Complexity of GDM}
A thorough analysis of the computational complexity is not included here; this is a short summary of the computational aspects involved. The main numerical component of GDM is computing $\nabla GD$. For a single iteration its complexity is $O(K \cdot N \cdot D^2)$. Our choice of $\rho$ requires a sorting procedure and is thus of order $O(N \cdot \log(N))$ operations for a single iteration. The initialization of the algorithm via ALC-type procedure~\cite{Ma07Compression} requires $O(n_1 \cdot N \cdot \log(N) \cdot D^2)$ operations. Also, the last genetic step has the following complexity $O(n_1 \cdot n_3 \cdot K \cdot D^2 \cdot N^2)$ (without taking advantage of incremental SVD). In theory, we can make the algorithm linear in the number of points $N$, by randomly initializing it, removing the genetic ``clean-up'' step and changing how we select our step size\footnote{This is assuming that we will not require more iterations to get close enough to the minimum that we can apply thresholding. In our experiments the number of needed iterations does not appear to grow with $N$, but we do not have any results to guarantee this.}. We have good numerical evidence, even with large $N$, that this can result in good accuracy and speed for artificial data. Regardless, for the values of $N$ in our application the algorithm is sufficiently fast and these additional steps help improving accuracy, especially for points which are nearby several clusters (whose percentage is not negligible when $N$ is small).

\section{Detecting and Rejecting Outliers with GDM}
\label{sec:Outlier Detection}
In practice, it turns out that the GDM algorithm described above is naturally robust to a small number of outliers (in that they do not tend to affect the classification of inliers), but no instruments were put in place for explicitly detecting or rejecting these outlying points. In this section, we introduce a modification to GDM that allows for explicit outlier detection and rejection. The guiding intuition is that an outlier has the property that if the true hybrid-linear structure is reflected in a partition, then no matter which group we assign the outlier to, it causes a significant increase in the empirical dimension of that group. This, in turn, results in a significant increase in global dimension. In other words, if we have a partition that reflects the true hybrid-linear structure of the data set, then there is no good place to put an outlier. If the algorithm was given the option of paying a fixed, low price for the right to ignore a given point, it would make sense for it to exercise this option on outliers, and only segment inliers.

We propose modifying the global dimension objective function, and the accompanying variational development in the following way:
\begin{equation}
\label{eqn:New Definition of GD}
\textrm{GD}(\M) = \alpha \Vert \M_{1,:} \Vert_1 + \Vert (\hat{d}_{\epsilon,2}, \hat{d}_{\epsilon,3}, ..., \hat{d}_{\epsilon,K+1})^T \Vert_p
\end{equation}
where:
\begin{equation}
\hat{d}_{\epsilon,k} = \hat{d}_\epsilon \left( \M_{k,1}\bv_1, \M_{k,2}\bv_2, ..., \M_{k,N}\bv_N \right).
\end{equation}

This modification adds an additional ``cluster'' to the problem (call it cluster 1), and we treat it differently than the others. Clusters $2$ through $K+1$ contribute to the global dimension in the same way that they did in the original development. Cluster $1$ contributes to the cost function the sum of the membership strengths of all data points to this cluster. This is the ``fuzzy assignment'' version of the following notion: we allow the algorithm to pay a fixed price, $\alpha$, for the right to ignore any particular data point (not assign it to any true cluster).

\subsection{Modification to GDM}
\label{subsec:Modification to GDM}
The proposed modification to the objective function only trivially changes the state space (now it is the product of $N$ $K+1$-dimensional probability simplexes, as opposed to $K$-dimensional simplexes). Thus, our method of projecting states onto the convex domain is effectively the same. The change to the objective function does mean that we must re-evaluate the gradient of global dimension. The computation is very similar to the unmodified version, and the result is:
\begin{equation}
\label{eqn:New Gradient}
{\partial GD \over \partial m_1^n} = \alpha m_1^n
\end{equation}
and, for all $k>1$
\begin{equation}
\label{eqn:New Gradient 2}
{\partial GD \over \partial m_k^n} = \V_{k(n,:)} \hat{d}_{\epsilon,k}^{p-1} \left( \hat{d}_{\epsilon,2}^p + ... + \hat{d}_{\epsilon,K+1}^p \right)^{{1 \over p}-1} \D_k \U_k^T \bv_n,
\end{equation}
where the notation and constants are as defined in \S\ref{sec:A Variational Algorithm}. Thus, the necessary modifications to GDM are:
\begin{enumerate}
\item Update the evaluation of the objective function $GD$ according to \eqref{eqn:New Definition of GD}.
\item Update the initialization of the state vector to include an outlier group.
\item Update the state projection routine to accommodate additional dimensions in domain.
\item Update the evaluation of $\nabla GD$ according to \ref{eqn:New Gradient} and \ref{eqn:New Gradient 2}.
\end{enumerate}

\subsection{Practical Implementations of Outlier Rejection}
\label{sec:Practical Methods of Outlier Detection}
We have described an idea for how to handle outliers, but it introduces a new parameter, $\alpha$. It is not immediately clear how one should choose this parameter, and how sensitive the results will be to it. In theory one would need to choose an outlier cost, $\alpha$, that is not so high that nothing is ever assigned to the outlier group, but not so low that large quantities of inliers are assigned to this group. The appropriate values would likely depend on multiple quantities, like intrinsic dimension, noise level, and distortion of the underlying subspaces. These are quantities that can vary not just between applications, but also from data set to data set for a single application. Applying the suggested modification exactly as proposed (and trying to ``tune'' this parameter) would therefore lead to an unreliable and unpredictable algorithm. We refer to this approach as GDM-Naive, and Figure \ref{fig: Triangles} illustrates why this method is unsound. Instead, we propose two variations of this method, which lead to more reliable solutions.

\begin{figure*}[t]
\captionsetup{margin=10pt,labelfont=bf}
	\centering
	\fbox{	
	\begin{minipage}[l]{.19\linewidth}
        \centering
        \includegraphics[width=\linewidth, clip=true, trim=0mm 0mm 0mm 0mm]{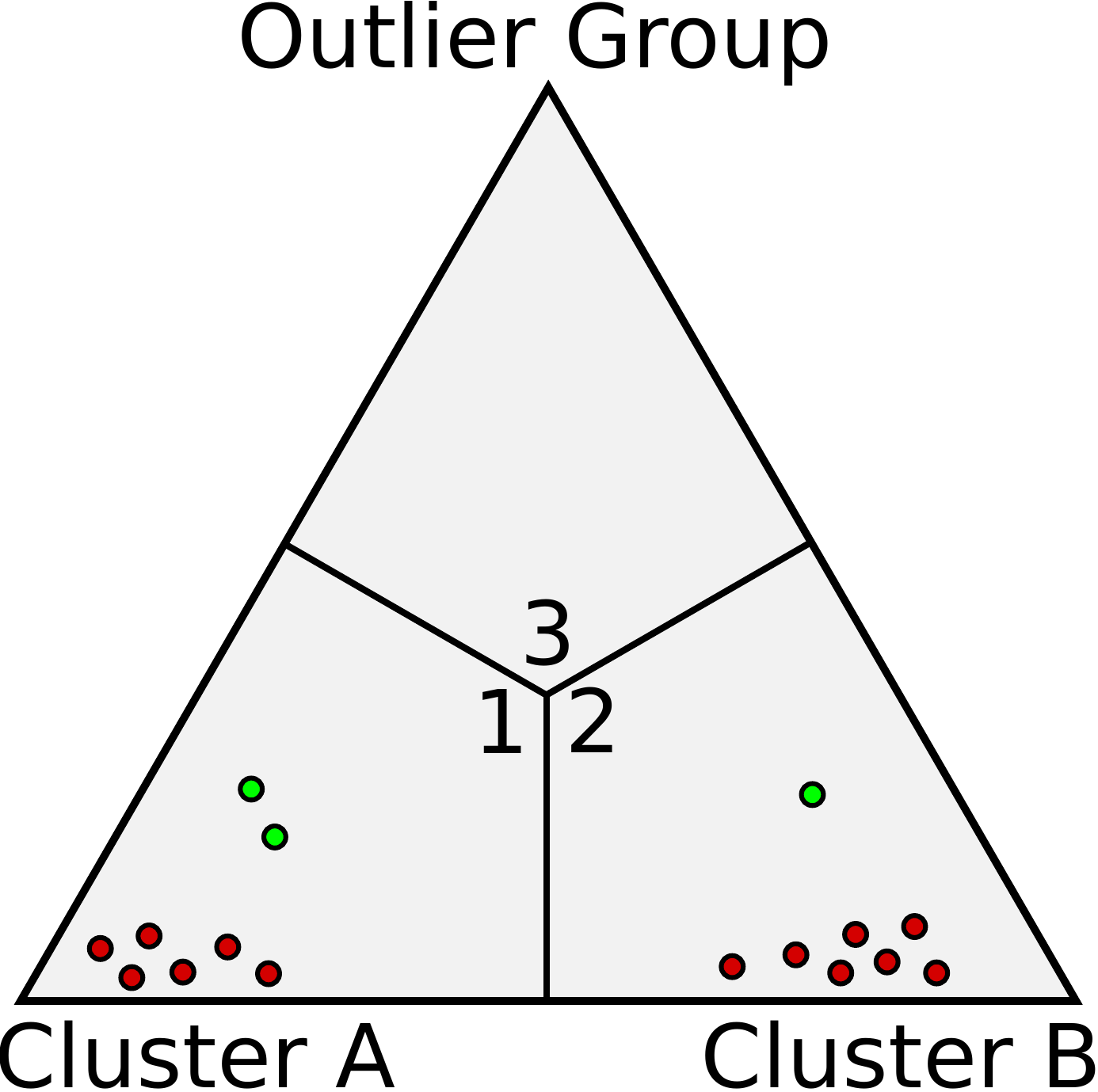}
        \caption*{Large $\alpha$: Nothing ends up in outlier group.\\}
    \end{minipage} \hspace{3mm}
    \begin{minipage}[l]{.19\linewidth}
        \centering
        \includegraphics[width=\linewidth, clip=true, trim=0mm 0mm 0mm 0mm]{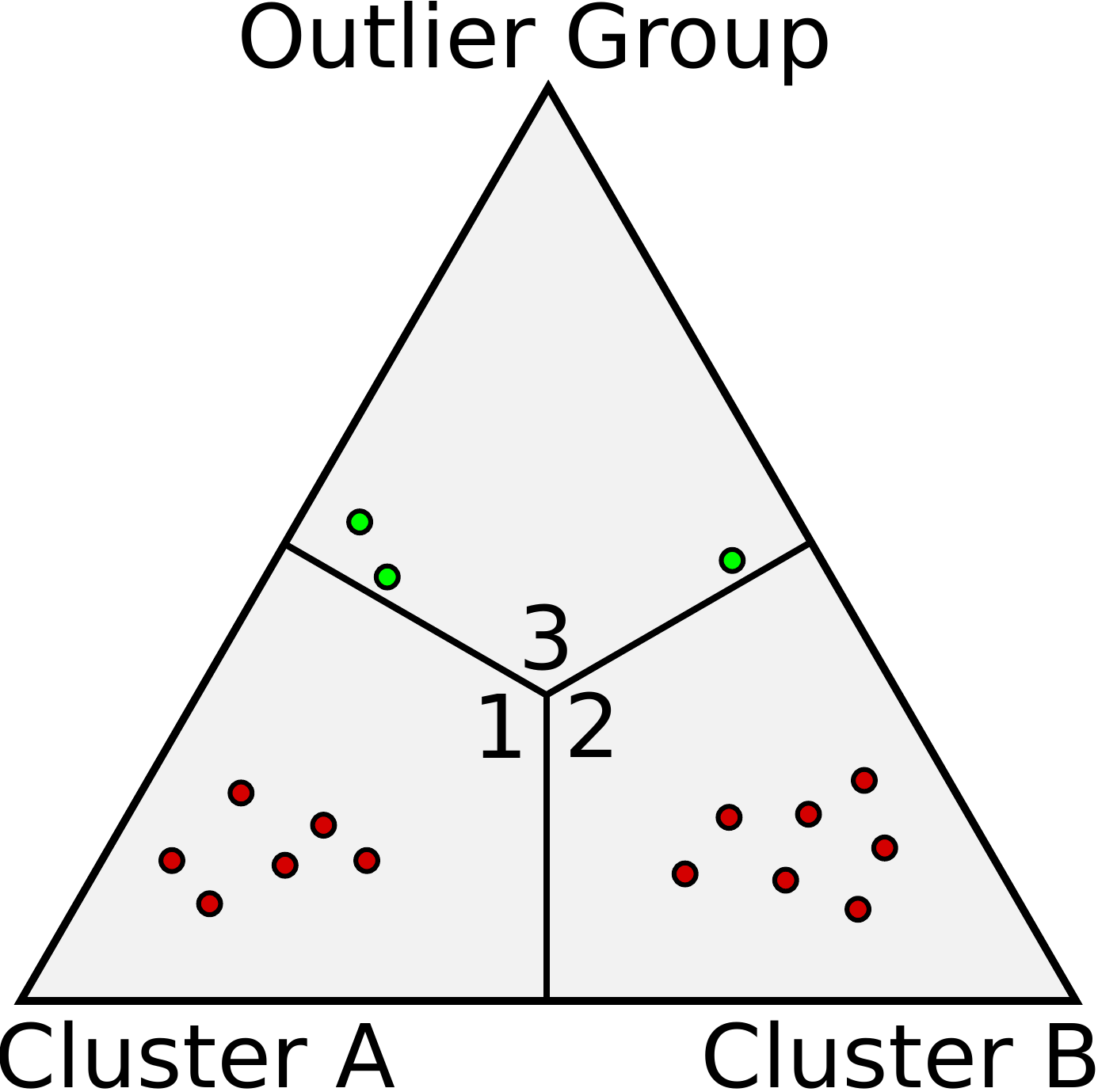}
        \caption*{Medium $\alpha$: Only outliers end up in outlier group.\\}
    \end{minipage} \hspace{3mm}
    \begin{minipage}[l]{.19\linewidth}
        \centering
        \includegraphics[width=\linewidth, clip=true, trim=0mm 0mm 0mm 0mm]{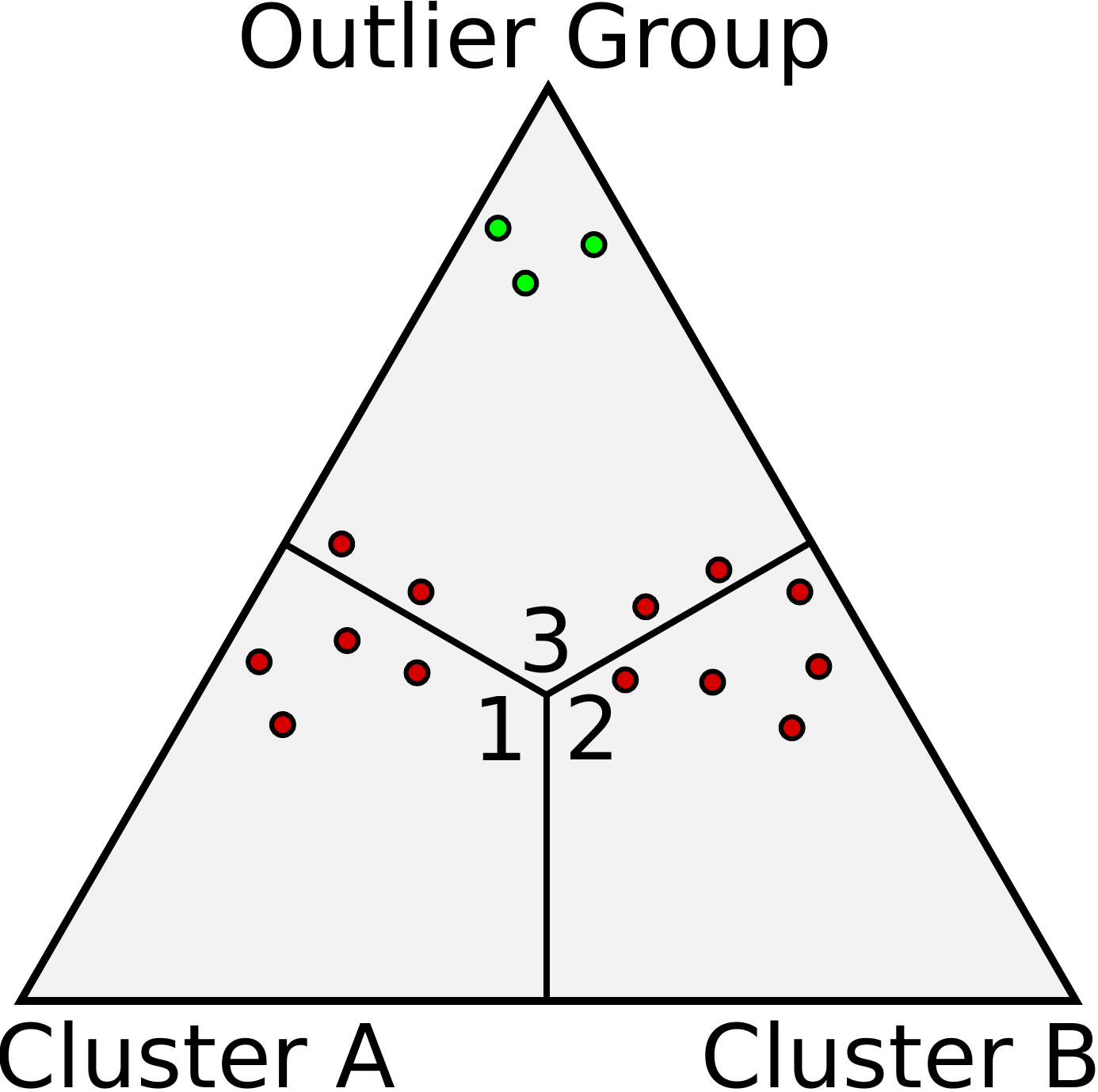}
        \caption*{Low $\alpha$: Some inliers end up in outlier group.\\}
    \end{minipage} \hspace{4mm}
    \begin{minipage}[l]{.34\linewidth}
        The three images here illustrate the problem with GDM-Naive. Each triangle represents the probability simplex containing the fuzzy assignment vectors for a fictitious data set. The fuzzy assignment for each point is plotted after many iterations of GDM. Points in red are inliers and points in green are outliers. The quantization regions (for the threshold step) are numbered 1-3. One can see that if $\alpha$ is not chosen correctly, points can be quantized into the wrong cluster. On the other hand, the outlier ranking of a point (the \# of points closer the outlier corner) is a more stable quantity. (Color figure online)
    \end{minipage}
    }
    \caption{Graphical depiction of the problem with GDM-Naive}
    \label{fig: Triangles}
\end{figure*}

\begin{enumerate}
\item \label{Method 1}{\bf GDM Known-Fraction:} Run the proposed algorithm with a fixed, low value of $\alpha$ (we use $\alpha=0.01$) but stop before the threshold step. Rank the data points according to their membership strengths to the outlier group. Remove a pre-set fraction of the data set (the part that most strongly affiliates with the outlier group). Continue with the classic (non-outlier version) of the variational algorithm on the surviving points only\footnote{We could skip this step and segment directly from the fuzzy assignment that we already have. Refining the membership matrix after removing the outliers is done to repair whatever damage the outliers may have done to the membership matrix before thresholding.} - this provides the inlier segmentation. The points that were removed are labelled outliers.
\item \label{Method 2}{\bf GDM Model-Reassign:} Run method \ref{Method 1} above (GDM Known-Fraction). Fit subspaces of appropriate dimension (round the empirical dimension) to each set in the resulting partition. Re-assign all points (including those that were decided to be outliers) according to their distances from each subspace. Call a point an outlier if it is more than some fixed distance, $\kappa$, from all of the subspaces.
\end{enumerate}

Each of the proposed methods handles the task of selecting $\alpha$, but introduces a new parameter. For method \ref{Method 1}, this is the percentage of the data set to throw out. For method \ref{Method 2}, the new parameter is the maximum distance a point can be from a subspace to be considered an inlier. Both of these parameters are more natural than selecting $\alpha$. In a noisy environment, one may have an idea, based on experiments, of what percentage of the data set will be outliers, or what the inlier modelling error tends to be. Additionally, when using the ``Model-Reassign'' method, one could find the average and variance of the residuals, $\mu$, and $\sigma^2$ respectively, when fitting subspaces to the inlier clusters. These quantities can be used to come up with a reasonable value of $\kappa$ for a given application ($\mu + r \sigma$ for some $r$). One could also find these values on a per-cluster basis and have a different outlier threshold for each cluster.

\section{Results on Real-World Data}
\label{sec:Results}
\subsection{Performance in the Absence of Outliers}
We tested the GDM algorithm on 2 motion segmentation databases. First, we used the outlier-free RAS database~\cite{rao_ijcv,AtevKSCC} and compared with many leading methods in 2-view segmentation. We noticed that some of the HLM methods performed better when using the linearly embedded point correspondences than with the nonlinear embedding. Therefore, in Table~\ref{tab:ras_trueK} we present each of the competing HLM algorithms twice. Where ``Linear'' appears, the algorithm was run on the feature trajectories in $\R^4$. Where ``Nonlinear'' appears, the algorithm was run on the Kronecker products (in $\R^9$) of the standard homogeneous coordinates of each feature correspondence. Figure \ref{fig:RAS_Figure} presents more details on the performance of the HLM methods with the nonlinear embedding, and Table \ref{Table:Runtimes for Clean RAS Comparison} gives the average runtimes of these methods. The other HLM methods we included are SCC~\cite{spectral_applied}, MAPA~\cite{mapa}, SSC~\cite{ssc09}, SLBF~\cite{LBF_journal12}, and LRR~\cite{lrr_short}. We also included two other successful methods for two-views (for which there was a code available online): RAS~\cite{rao_ijcv} and HOSC~\cite{higher-order}. Algorithm parameters and our experiment procedure are detailed in \S\ref{sec:Experiment Setup}.

\begin{figure}[h]
\captionsetup{margin=10pt,labelfont=bf}
\centering
\subfloat[SCC]{
	\includegraphics[width=.3\linewidth, clip=true, trim=12mm 12mm 12mm 12mm]{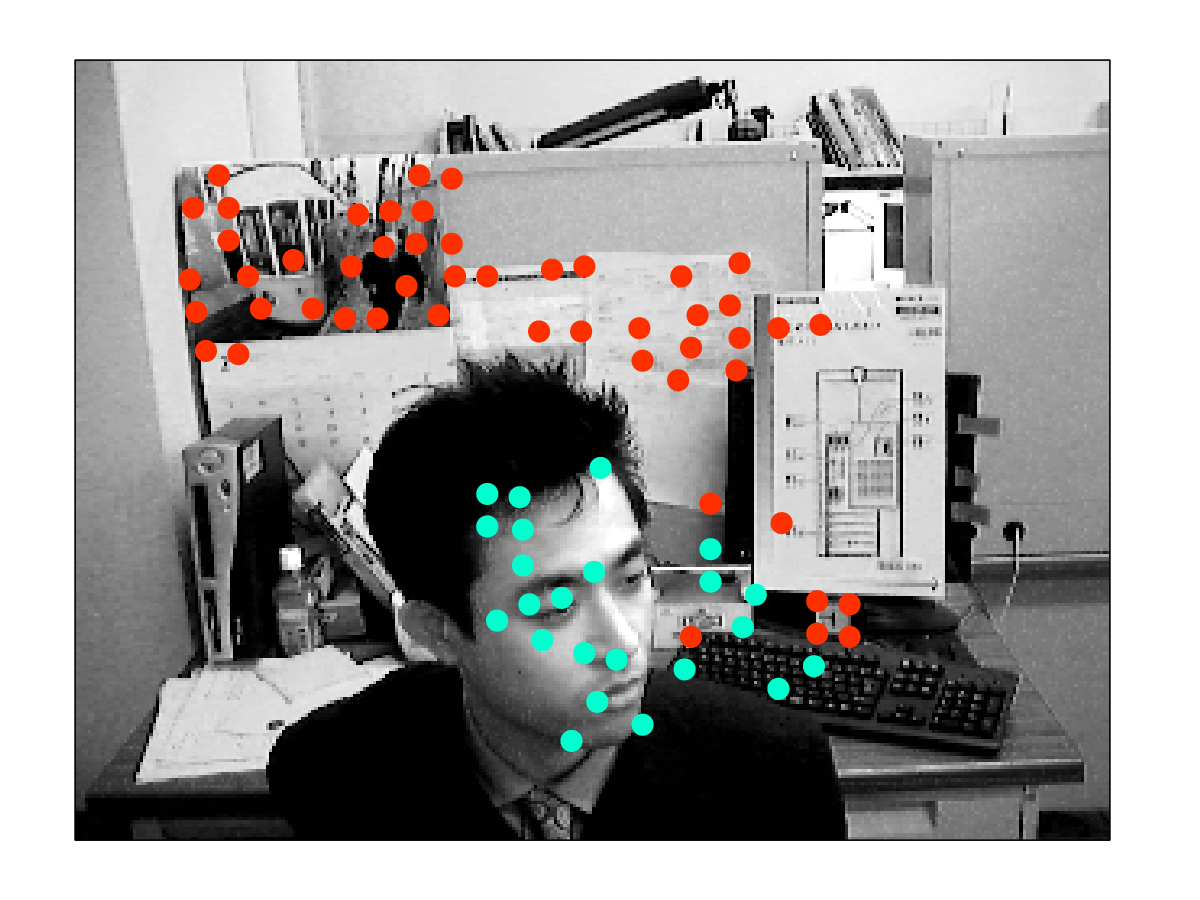}
}
\subfloat[GDM]{
	\includegraphics[width=.3\linewidth, clip=true, trim=12mm 12mm 12mm 12mm]{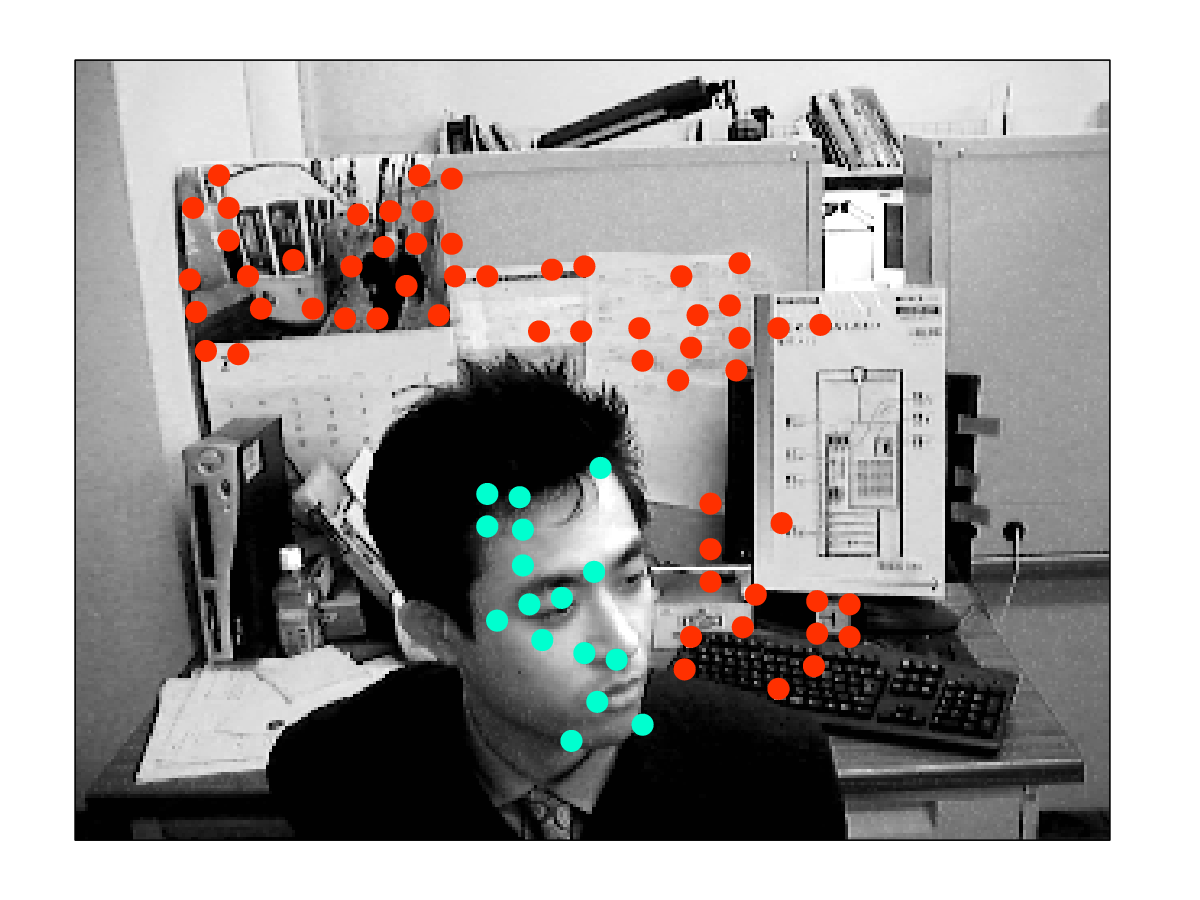}
}
\subfloat[MAPA]{
	\includegraphics[width=.3\linewidth, clip=true, trim=12mm 12mm 12mm 12mm]{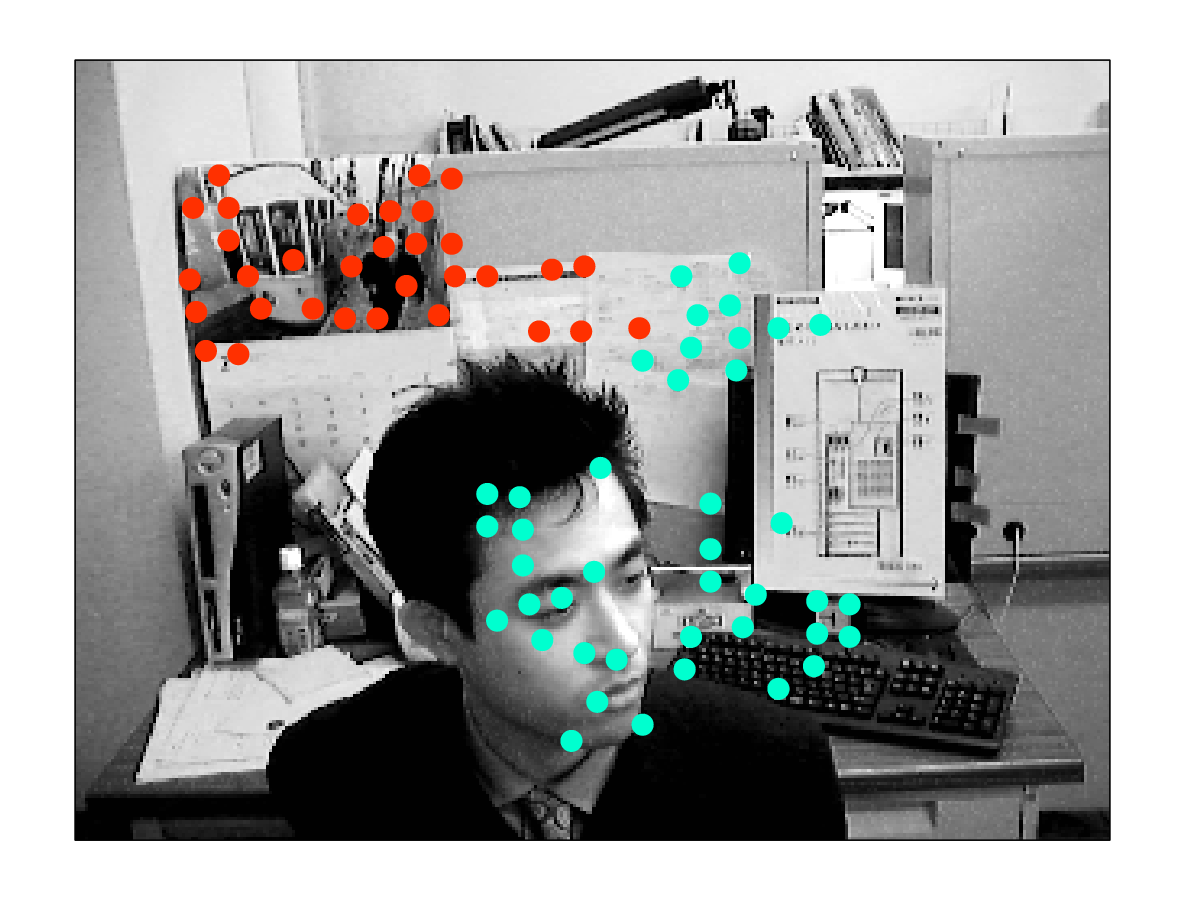}
}
\caption{\label{fig:RAS_Demonstration}Clustering by SCC, GDM, and MAPA on file 6 of the outlier-free RAS database. (Color figure online)}
\end{figure}

\begin{table*}[htb]
	\captionsetup{margin=10pt,font=small,labelfont=bf}
	\caption{\label{tab:ras_trueK} Misclassification Rates (given as \% Error) on the outlier-free RAS database.}
	\centering
	\scriptsize{
	\tabcolsep=0.17cm
	\begin{tabular}{ l  l | c c c c c c c c c c c c c | c | c |}
	\cline{3-17}
	& & \multicolumn{13}{|c|}{\textbf{File Number}} & \multicolumn{1}{|c|}{\multirow{2}{*}{\bf{Average}}} & \multicolumn{1}{|c|}{\bf{Average}} \\
	\cline{3-15}
	& & \multicolumn{1}{|c|}{1}  & \multicolumn{1}{|c|}{2}  & \multicolumn{1}{|c|}{3}  & \multicolumn{1}{|c|}{4} & \multicolumn{1}{|c|}{5}
	  & \multicolumn{1}{|c|}{6}  & \multicolumn{1}{|c|}{7}  & \multicolumn{1}{|c|}{8}  & \multicolumn{1}{|c|}{9} & \multicolumn{1}{|c|}{10}
	  & \multicolumn{1}{|c|}{11} & \multicolumn{1}{|c|}{12} & \multicolumn{1}{|c|}{13} & \multicolumn{1}{|c|}{}  & \multicolumn{1}{|c|}{\bf{w/o File \#8}} \\
	\hline
	\multicolumn{1}{|c|}{\multirow{15}{*}{\begin{sideways}\textbf{Method/Embedding}\end{sideways}}}
	                         & \multicolumn{1}{|l|}{GDM Nonlinear} & \textbf{  0.85} & \textbf{  0.00} &   1.57 &   0.65 & \textbf{  0.00} & \textbf{  0.00} & \textbf{  0.00} &  12.76 & \textbf{  0.00} & \textbf{  0.00} & \textbf{  0.00} & \textbf{  0.00} & \textbf{  0.00} &   1.22 & \textbf{  0.26}\\ \cline{2-17}
	\multicolumn{1}{|c|}{}   & \multicolumn{1}{|l|}{SCC Linear} & \textbf{  0.85} & \textbf{  0.00} &   1.18 &   0.65 & \textbf{  0.00} &   1.37 & \textbf{  0.00} &   1.42 &   0.39 & \textbf{  0.00} & \textbf{  0.00} &   1.01 & \textbf{  0.00} & \textbf{  0.53} &   0.45\\ \cline{2-17}
	\multicolumn{1}{|c|}{}   & \multicolumn{1}{|l|}{SCC Nonlinear} & \textbf{  0.85} & \textbf{  0.00} &  24.41 & \textbf{  0.00} & \textbf{  0.00} &  19.18 & \textbf{  0.00} & \textbf{  0.00} & \textbf{  0.00} &  13.97 &   5.36 &   0.84 &   1.10 &   5.05 &   5.48\\ \cline{2-17}
	\multicolumn{1}{|c|}{}   & \multicolumn{1}{|l|}{MAPA Linear} & \textbf{  0.85} &   3.65 &   1.18 &   0.65 & \textbf{  0.00} &  13.70 &  15.97 &   1.29 & \textbf{  0.00} & \textbf{  0.00} & \textbf{  0.00} &   0.67 &   3.30 &   3.17 &   3.33\\ \cline{2-17}
	\multicolumn{1}{|c|}{}   & \multicolumn{1}{|l|}{MAPA Nonlinear} & \textbf{  0.85} &  20.55 &  21.65 &   0.65 & \textbf{  0.00} &  21.92 &   6.25 &   7.73 & \textbf{  0.00} &  13.97 &   1.43 &   0.34 &   3.30 &   7.59 &   7.57\\ \cline{2-17}
	\multicolumn{1}{|c|}{}   & \multicolumn{1}{|l|}{SSC Linear} &   1.69 &  18.26 & \textbf{  0.79} &   1.94 & \textbf{  0.00} & \textbf{  0.00} &   6.25 &  32.22 & \textbf{  0.00} & \textbf{  0.00} &  14.64 &   1.35 &   4.40 &   6.27 &   4.11\\ \cline{2-17}
	\multicolumn{1}{|c|}{}   & \multicolumn{1}{|l|}{SSC Nonlinear} &   1.27 & \textbf{  0.00} &  22.44 &   0.65 & \textbf{  0.00} &  21.92 & \textbf{  0.00} &   9.02 & \textbf{  0.00} &  13.97 &   9.29 &  12.12 &   6.59 &   7.48 &   7.35\\ \cline{2-17}
	\multicolumn{1}{|c|}{}   & \multicolumn{1}{|l|}{SLBF Linear} & \textbf{  0.85} &   0.46 &   1.18 &   0.65 & \textbf{  0.00} & \textbf{  0.00} & \textbf{  0.00} &   0.26 & \textbf{  0.00} & \textbf{  0.00} & \textbf{  0.00} &   0.67 &   4.40 &   0.65 &   0.68\\ \cline{2-17}
	\multicolumn{1}{|c|}{}   & \multicolumn{1}{|l|}{SLBF Nonlinear} & \textbf{  0.85} & \textbf{  0.00} &   5.12 &   1.94 & \textbf{  0.00} &  19.18 & \textbf{  0.00} &  10.57 & \textbf{  0.00} &  13.97 & \textbf{  0.00} &   1.68 &  14.29 &   5.20 &   4.75\\ \cline{2-17}
	\multicolumn{1}{|c|}{}   & \multicolumn{1}{|l|}{LRR Linear} &   5.08 &  24.66 &   1.18 &   2.58 &   2.38 &   2.74 & \textbf{  0.00} &  29.12 & \textbf{  0.00} & \textbf{  0.00} &   8.93 &  14.81 &  18.68 &   8.47 &   6.75\\ \cline{2-17}
	\multicolumn{1}{|c|}{}   & \multicolumn{1}{|l|}{LRR Nonlinear} &   1.27 &   9.13 &   2.76 &   1.94 & \textbf{  0.00} & \textbf{  0.00} &   3.47 &   3.61 & \textbf{  0.00} & \textbf{  0.00} &   9.64 &  18.18 &   2.20 &   4.02 &   4.05\\ \cline{2-17}
	\multicolumn{1}{|c|}{}   & \multicolumn{1}{|l|}{RAS} &  11.65 & \textbf{  0.00} &   2.56 &   9.68 &  16.19 &  26.03 &  26.74 &  11.21 &   3.28 &  13.97 &   3.21 &   2.36 &   6.59 &  10.27 &  10.19\\ \cline{2-17}
	\multicolumn{1}{|c|}{}   & \multicolumn{1}{|l|}{HOSC d=2} & \textbf{  0.85} & \textbf{  0.00} &  24.41 &   1.61 & \textbf{  0.00} & \textbf{  0.00} & \textbf{  0.00} &  22.94 & \textbf{  0.00} & \textbf{  0.00} & \textbf{  0.00} & \textbf{  0.00} &   2.20 &   4.00 &   2.42\\ \cline{2-17}
	\multicolumn{1}{|c|}{}   & \multicolumn{1}{|l|}{HOSC d=3} &   1.27 &  23.74 &  24.41 &   3.23 & \textbf{  0.00} &  19.18 &  12.15 &  19.59 &  23.75 & \textbf{  0.00} &   1.43 &   1.01 &  17.58 &  11.33 &  10.65\\
     \hline
	\end{tabular}
	}
     \label{RASResults}
\end{table*}

\begin{figure}[h]
\captionsetup{margin=10pt,labelfont=bf}
\centering
\includegraphics[width=\linewidth, clip=true, trim=7mm 20mm 7mm 20mm]{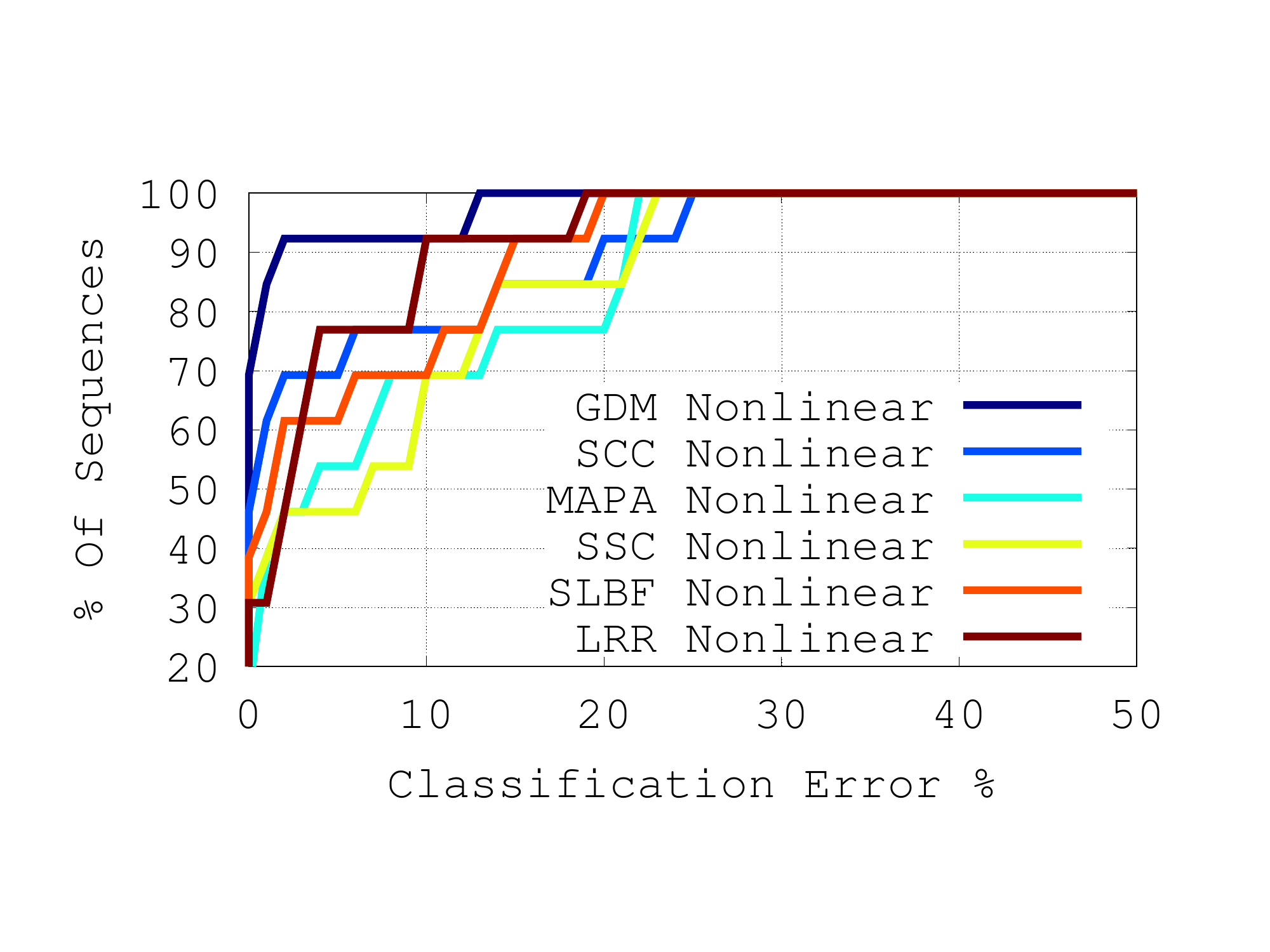}
\caption{\label{fig:RAS_Figure}GDM is compared against other HLM methods on the nonlinear 2-view embedding of the outlier-free RAS database. (Color figure online)}
\end{figure}

\begin{table}[htb]
	\captionsetup{margin=10pt,labelfont=bf}
	\caption{Average runtimes (per file) of HLM-based methods on non-linearly embedded (outlier-free) RAS data.}
	\centering
	\tabcolsep=0.20cm
	\begin{tabular}{ l l | c |}
	\cline{3-3}
	& & \multicolumn{1}{|c|}{Runtime (seconds)} \\ \hline
	\multicolumn{1}{|c|}{\multirow{6}{*}{\begin{sideways}\bf{Method}\end{sideways}}}
	                         & \multicolumn{1}{|l|}{GDM}  & 12.7 \\ \cline{2-3}
	\multicolumn{1}{|c|}{}   & \multicolumn{1}{|l|}{SCC}  & 2.3  \\ \cline{2-3}
	\multicolumn{1}{|c|}{}   & \multicolumn{1}{|l|}{MAPA} & 5.6  \\ \cline{2-3}
	\multicolumn{1}{|c|}{}   & \multicolumn{1}{|l|}{SSC}  & 89.5 \\ \cline{2-3}
    \multicolumn{1}{|c|}{}   & \multicolumn{1}{|l|}{SLBF} & 4.0  \\ \cline{2-3}
	\multicolumn{1}{|c|}{}   & \multicolumn{1}{|l|}{LRR}  & 0.8  \\ \cline{2-3}
	\hline
	\end{tabular}
   	\label{Table:Runtimes for Clean RAS Comparison}
\end{table}

From Table \ref{tab:ras_trueK} and Figure \ref{fig:RAS_Figure}, we can see that GDM performs very competitively on this database. There is only a single file ($\#8$) on which GDM exhibits significant error. This file contains features from two bent magazines as well as a rigid background. Since the bent magazines are clearly non-rigid, our model assumptions are not met (see Fig. \ref{fig:File 8}). There were two methods in the comparison that had a lower average misclassification error than GDM (``SCC Linear'' and ``SLBF Linear''). This is because they perform significantly better on file ($\#8$). Both of these are spectral methods, accompanied by the linear embedding, and are therefore better able to handle the manifold structure that results from the non-rigidity of the objects in this file. Amongst the other files however, GDM performs better on average than both of these two methods (see the last column of Table \ref{tab:ras_trueK}). Comparing just the HLM-based methods on the nonlinearly-embedded data, GDM performs better than any other method, with the most perfect classifications and the fewest number of files with significant errors. Figure \ref{fig:RAS_Figure} more clearly emphasizes this superb performance amongst methods using the nonlinear embedding.

\begin{figure}[h]
\captionsetup{margin=10pt,labelfont=bf}
\centering
\subfloat[Frame 1]{
	\includegraphics[width=.45\linewidth, clip=true, trim=12mm 9mm 9mm 11mm]{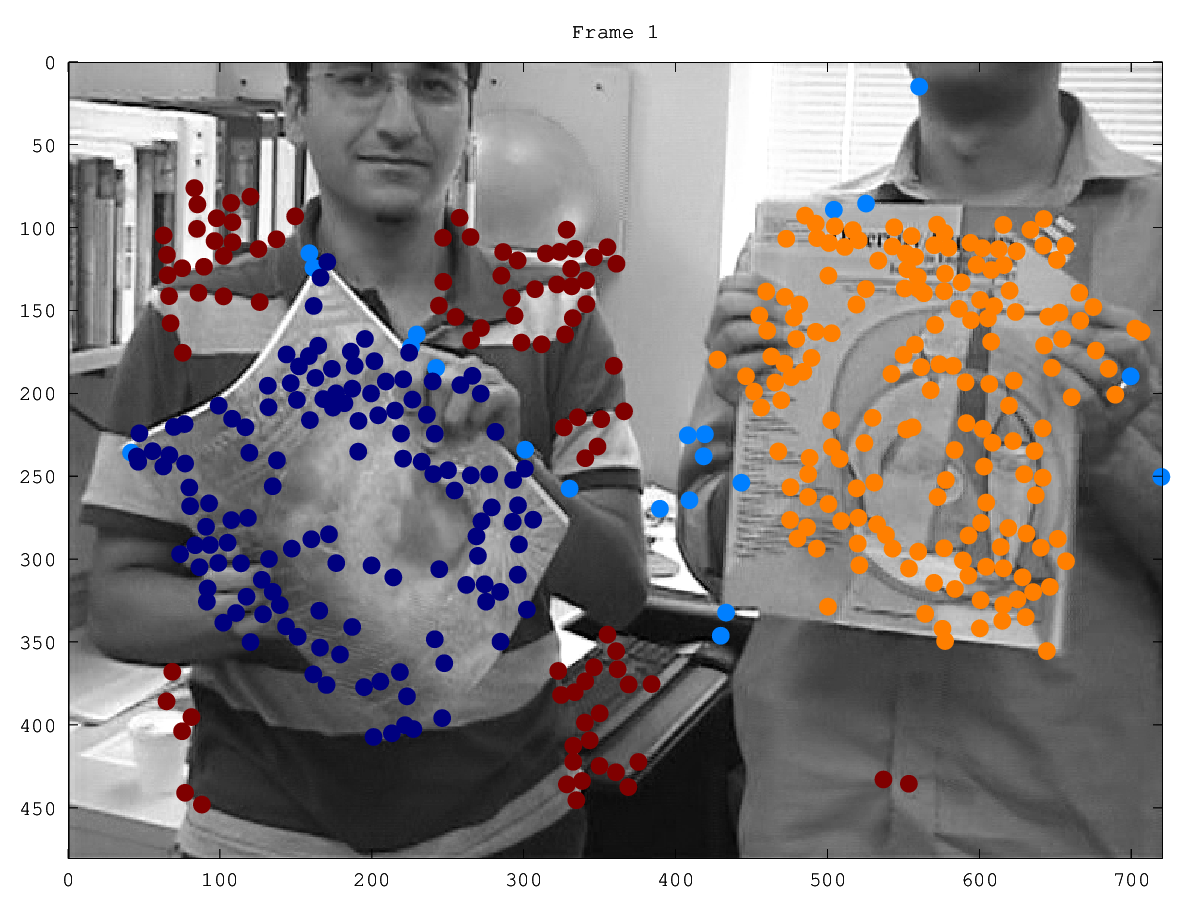}
}
\subfloat[Frame 2]{
	\includegraphics[width=.45\linewidth, clip=true, trim=12mm 9mm 9mm 11mm]{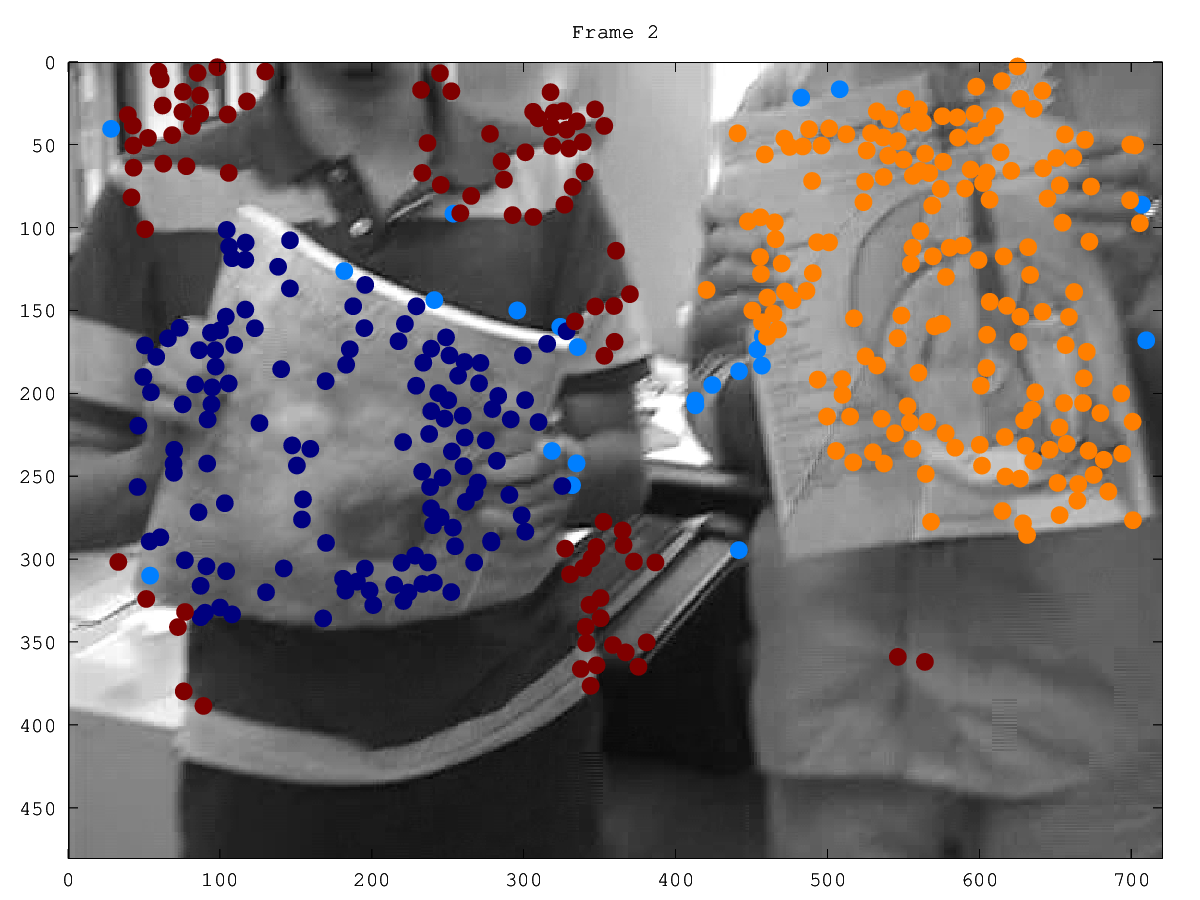}
}
\caption{\label{fig:File 8}File 8 in the RAS database. This is a problematic file because the two magazines in the scene appear to undergo a non-rigid transformation between the two frames. Point correspondences are colored according to ground-truth segmentation. (Color figure online)}
\end{figure}

We also performed experiments on the Hopkins155 database~\cite{Tron07abenchmark}. For 2-view segmentation we extracted the first and last frame of each sequence and performed 2-view segmentation on the nonlinear embedding (in $\R^9$) of the data. For comparison, we demonstrate the results of some other HLM algorithms on this embedded data: MAPA~\cite{mapa}, SCC-MS~\cite{spectral_applied,LBF_journal12} and SLBF-MS~\cite{LBF_journal12}. We also supply results for a few state-of-the-art HLM methods on the full $n$-view feature trajectories. For these $n$-view results we chose in this table the best methods on Hopkins155 we are aware of, which do not require careful tuning with parameters: SSC~\cite{ssc09} and SLBF-MS~\cite{LBF_journal12}. We also include the reference (REF) results~\cite{Tron07abenchmark}. REF finds the best linear models (via least squares approximation) for each cluster of embedded points (given the ground truth segmentation), and then finds new clusters by assigning points to the models they best agree with. For GDM on this database, it was necessary to increase the number of random initializations ($n_1$ in Algorithm 1) to achieve reliable convergence (we changed it from 10 to 30). From Table \ref{tab:hopkins_trueK} we see that GDM outperforms the other 2-view methods (although SCC matches or nearly matches its performance in some categories). We remark that we also tested a genetic algorithm for minimizing the global dimension and it achieved even more accurate results, however, we do not include it here since it is not as fast as GDM.

\begin{table*}[htb]
\captionsetup{margin=10pt,labelfont=bf}
\caption{\label{tab:hopkins_trueK} The mean and median percentage of misclassified points for two-motions and three-motions in Hopkins 155 database with comparisons to state-of-the-art $n$-views. Winning results amongst the 2-view methods are bold-faced in each category.}
\centering
\subfloat{
     \tabcolsep=0.35cm
     \rowcolors{5}{white}{lightgray}
     \begin{tabular}{c|l||c|c||c|c||c|c||c|c|}
          \cline{2-10}
          & & \multicolumn{2}{c||}{Checker} & \multicolumn{2}{c||}{Traffic} & \multicolumn{2}{c||}{Articulated} & \multicolumn{2}{c|}{All}\\
          \cline{3-10}
          & \raisebox{1.5ex}[0pt]{\normalsize{2-motion}} & Mean & Median & Mean & Median & Mean & Median & Mean & Median\\
          \hline
          \multicolumn{1}{|c|}{\multirow{4}{*}{\begin{sideways}\textbf{$2$-view}\end{sideways}}}
                                 & GDM              & \bf{2.79}  & \bf{0.00}  & \bf{1.78} & \bf{0.00} & \bf{2.66} & \bf{0.00} & \bf{2.51}  & \bf{0.00} \\
          \multicolumn{1}{|c|}{} & MAPA             & 12.85      & 14.07      & 6.49      & 6.93      & 7.15      & 5.33      & 10.69      & 10.03     \\
          \multicolumn{1}{|c|}{} & SCC (d=7)        & \bf{2.79}  & \bf{0.00}  & 1.97      & \bf{0.00} & 3.42      & \bf{0.00} & 2.64       & \bf{0.00} \\
          \multicolumn{1}{|c|}{} & SLBF (d=6)       & 8.18       & 1.39       & 3.98      & 0.53      & 4.73      & 0.40      & 6.78       & 1.11      \\
          \hline
          \multicolumn{1}{|c|}{\multirow{3}{*}{\begin{sideways}\textbf{$n$-view}\end{sideways}}}
                                 & SLBF-MS (2$F$,3) & 1.28 & 0.00 & 0.21 & 0.00 & 0.94 & 0.00 & 0.98 & 0.00 \\
          \multicolumn{1}{|c|}{} & SSC-N (4$K$,3)   & 1.29 & 0.00 & 0.29 & 0.00 & 0.97 & 0.00 & 1.00 & 0.00 \\
          \multicolumn{1}{|c|}{} & REF              & 2.76 & 0.49 & 0.30 & 0.00 & 1.71 & 0.00 & 2.03 & 0.00 \\
          \hline
     \end{tabular}
} \\
\subfloat{
     \tabcolsep=0.35cm
     \rowcolors{5}{white}{lightgray}
     \begin{tabular}{c|l||c|c||c|c||c|c||c|c|}
          \cline{2-10}
          & & \multicolumn{2}{c||}{Checker} & \multicolumn{2}{c||}{Traffic} & \multicolumn{2}{c||}{Articulated} & \multicolumn{2}{c|}{All}\\
          \cline{3-10}
          & \raisebox{1.5ex}[0pt]{\normalsize{3-motion}}  & Mean & Median & Mean & Median & Mean & Median & Mean & Median\\
          \hline
          \multicolumn{1}{|c|}{\multirow{4}{*}{\begin{sideways}\textbf{$2$-view}\end{sideways}}}
                                 & GDM              & \bf{5.37} & \bf{3.23} & \bf{4.23} & \bf{2.69} & 5.32      & 5.32      & \bf{5.14} & \bf{3.13} \\
          \multicolumn{1}{|c|}{} & MAPA             & 21.89     & 19.49     & 13.15     & 13.04     & 9.04      & 9.04      & 19.41     & 18.09     \\
          \multicolumn{1}{|c|}{} & SCC (d=7)        & 8.05      & 5.85      & 4.67      & 5.45      & 5.85      & 5.85      & 7.25      & 5.45      \\
          \multicolumn{1}{|c|}{} & SLBF (d=6)       & 14.08     & 12.80     & 7.93      & 6.75      & \bf{4.79} & \bf{4.79} & 12.32     & 9.57      \\
          \hline
          \multicolumn{1}{|c|}{\multirow{3}{*}{\begin{sideways}\textbf{$n$-view}\end{sideways}}}
                                 & SLBF-MS (2$F$,3) & 3.33  & 0.39 & 0.24 & 0.00 & 2.13 & 2.13 & 2.64 & 0.22 \\
          \multicolumn{1}{|c|}{} & SSC-N (4$K$,3)   & 3.22  & 0.29 & 0.53 & 0.00 & 2.13 & 2.13 & 2.62 & 0.22 \\
          \multicolumn{1}{|c|}{} & REF              & 6.28  & 5.06 & 1.30 & 0.00 & 2.66 & 2.66 & 5.08 & 2.40 \\
          \hline
     \end{tabular}
}
\end{table*}

It is also interesting to note that our results for 2-views are comparable to the reference results with $n$-views. That is, the results of GDM are the best one can expect with pure linear modeling given many views and assuming an affine camera model. GDM for $n$-views gave comparable results and we thus did not include it. On the other hand, both SLBF-MS and SSC-N are able to obtain better results with $n$-views and this may be because their machinery of spectral clustering (together with good choices of spectral weights) allows them to take into account some of the manifold structure and nearness of points (information beyond linear modeling).

\subsection{Performance in the Presence of Outliers}
We tested the methods suggested in \S\ref{sec:Practical Methods of Outlier Detection} on the outlier-corrupted RAS database~\cite{rao_ijcv}. The performance of classic GDM (no outlier rejection machinery) is also presented on this database, as is the performance of GDM on the corresponding outlier-free database (for comparison purposes). We also show results from three competing methods for segmenting motion with outliers: RAS~\cite{rao_ijcv}, HOSC~\cite{higher-order}, and LRR~\cite{lrr_short,lrr_long} with outlier rejection performed by identifying the largest columns of $\mathbf{E}$, as suggested in \cite[pg. 9]{lrr_long}. The details of this experiment, including parameter values, are given in \S\ref{sec:Experiment Setup}.

It is non-trivial to fairly compare different algorithms in the presence of outliers. Each method generally has at least one parameter for controlling how it handles outliers. This parameter balances the desire for a high outlier detection rate with a desire for a low false alarm rate (these two quantities are invariably correlated). Using any popular metric for evaluating segmentation accuracy (like misclassification rate for true inliers\footnote{``True inliers'' are points that are inliers according to ground truth.}), the performance of each algorithm will depend substantially on its outlier handling parameter. In general terms, if an algorithm is allowed to discard points as outliers more freely, then the accuracy on the surviving points will improve. Thus, if one method is more conservative than another in discarding points as outliers, the results will likely be skewed in favor of one method over the other. It is therefore important when looking at segmentation accuracy to think in terms of accuracy for a given \emph{true positive rate} (TPR) and \emph{false positive rate} (FPR):

\begin{align*}
\text{TPR} &= {{\text{\# of outliers that were identified as outliers}} \over {\text{\# of outliers in dataset}}} * 100,\\[0.25cm]
\text{FPR} &= {{\text{\# of inliers that were identified as outliers}} \over {\text{\# of inliers in dataset}}} * 100.\\
\end{align*}

There are two aspects of these algorithms we wish to compare. The first is outlier detection performance (how good is each method at distinguishing between inliers and outliers). The second is segmentation performance, where we evaluate how good each method is at segmenting motions in the presence of outliers.

To compare the outlier detection performance of multiple methods, a common tool is the ROC curve, which parametrically plots the TPR vs. FPR as a function of the outlier parameter for a method. A ``random classifier'' that randomly labels points as inliers or outliers will have an ROC curve lying along the line $\text{TPR}=\text{FPR}$. An ideal classifier will follow the line $\text{TPR}=1$. Hence, methods can be compared by seeing which ROC curve is highest over the broadest range of FPRs (or over the FPRs one is interested in). The ROC curves for GDM (using the Model-Reassign outlier detection method and varying $\kappa$), LRR (by varying $\lambda$), RAS (by varying ``outlierFraction''), and HOSC (by varying $\alpha$), are presented in Fig.~\ref{fig:ROC_Figure}.

\begin{table*}[htb]
	\captionsetup{margin=10pt,font=small,labelfont=bf}
	\caption{Misclassification Rates (given as \% Error) of inliers on the RAS database. All but `Classic GDM - clean' are misclassification rates when run on the outlier-corrupted datasets. `GDM - clean' gives the performance of the unmodified GDM algorithm, when run on the outlier-removed datasets (included as a reference).}
	\centering
	\scriptsize{
	\tabcolsep=0.15cm
	\begin{tabular}{ l  l | c c c c c c c c c c c c c | c | c |}
	\cline{3-17}
	& & \multicolumn{13}{|c|}{\bf{File Number}} & \multicolumn{1}{|c|}{\multirow{2}{*}{\bf{Average}}} & \multicolumn{1}{|c|}{\bf{Average}} \\
	\cline{3-15}
	& & \multicolumn{1}{|c|}{1}  & \multicolumn{1}{|c|}{2}  & \multicolumn{1}{|c|}{3}  & \multicolumn{1}{|c|}{4} & \multicolumn{1}{|c|}{5}
	  & \multicolumn{1}{|c|}{6}  & \multicolumn{1}{|c|}{7}  & \multicolumn{1}{|c|}{8}  & \multicolumn{1}{|c|}{9} & \multicolumn{1}{|c|}{10}
	  & \multicolumn{1}{|c|}{11} & \multicolumn{1}{|c|}{12} & \multicolumn{1}{|c|}{13} & \multicolumn{1}{|c|}{}  & \multicolumn{1}{|c|}{\bf{w/o File \#8}} \\
	\hline
	\multicolumn{1}{|c|}{\multirow{6}{*}{\begin{sideways}\bf{Method}\end{sideways}}}
	                         & \multicolumn{1}{|l|}{GDM - Model-Reassign} &   2.97 & \textbf{  0.00} &   4.33 & \textbf{  1.29} &   0.95 & \textbf{  0.00} & \textbf{  0.00} &  12.63 & \textbf{  0.00} &   6.62 & \textbf{  0.00} & \textbf{  2.02} &  17.58 & \textbf{  3.72} & \textbf{  2.98}\\ \cline{2-17}
	\multicolumn{1}{|c|}{}   & \multicolumn{1}{|l|}{GDM - Classic} & \textbf{  0.85} & \textbf{  0.00} & \textbf{  1.57} &  32.26 & \textbf{  0.00} & \textbf{  0.00} & \textbf{  0.00} &  22.94 & \textbf{  0.00} & \textbf{  0.00} &  22.14 &   8.75 &  16.48 &   8.08 &   6.84\\ \cline{2-17}
	\multicolumn{1}{|c|}{}   & \multicolumn{1}{|l|}{RAS} &  19.49 &   5.02 &   1.97 &   5.81 &  15.71 &  23.29 &  25.00 &  11.86 &   2.32 &  13.97 &  12.14 &  18.18 &  21.98 &  13.60 &  13.74\\ \cline{2-17}
	\multicolumn{1}{|c|}{}   & \multicolumn{1}{|l|}{LRR} &   4.24 &  20.55 &  22.83 &   7.10 &   7.14 &   8.22 &  18.75 &  34.54 &   2.32 &  27.21 &   8.57 &  11.78 &  25.27 &  15.27 &  13.67\\ \cline{2-17}
	\multicolumn{1}{|c|}{}   & \multicolumn{1}{|l|}{HOSC (d=2)} &  11.02 &  22.37 &  16.54 &  33.55 &  10.95 &   2.74 &  11.11 & \textbf{ 11.34} &   3.09 &  13.97 &  36.79 &  66.67 & \textbf{  8.79} &  19.15 &  19.80\\
     \cline{2-17} \rowcolor{lightgray} \multicolumn{1}{|c|}{} & \multicolumn{1}{|l|}{GDM - clean} & 0.85 & 0.00 & 1.57 & 0.65 & 0.00 & 0.00 & 0.00 & 12.76 & 0.00 & 0.00 & 0.00 & 0.00 & 0.00 & 1.22 & 0.26 \\
	\hline
	\end{tabular}
	}
     \label{Table:InlierMis-Rates}
\end{table*}

GDM was again run using the nonlinear embedding of the data. HOSC was run with the linear embedding and LRR was run with the nonlinear embedding since these were the cases that yielded the best performance in the outlier-free tests for each algorithm (see \S\ref{sec:Experiment Setup} for more details). From Fig.~\ref{fig:ROC_Figure} we can see that GDM is very competitive at detecting outliers on this database. At low FPRs GDM yields comparatively excellent performance. At higher FPRs HOSC has a moderate advantage at outlier detection v.s.~GDM. However, it will be seen later (Table \ref{Table:InlierMis-Rates}) that HOSC is not competitive at segmentation in the presence of outliers. Furthermore, the presented HOSC results were prepared using $d=2$ (see \S\ref{sec:Experiment Setup}), instead of $d=3$ as argued for by its authors. Using $d=3$ gave worse results and made the algorithm take an extremely long time to execute.

\begin{figure}[thb]
\captionsetup{margin=10pt,labelfont=bf}
\centering
\includegraphics[width=\linewidth, clip=true, trim=20mm 0mm 20mm 5mm]{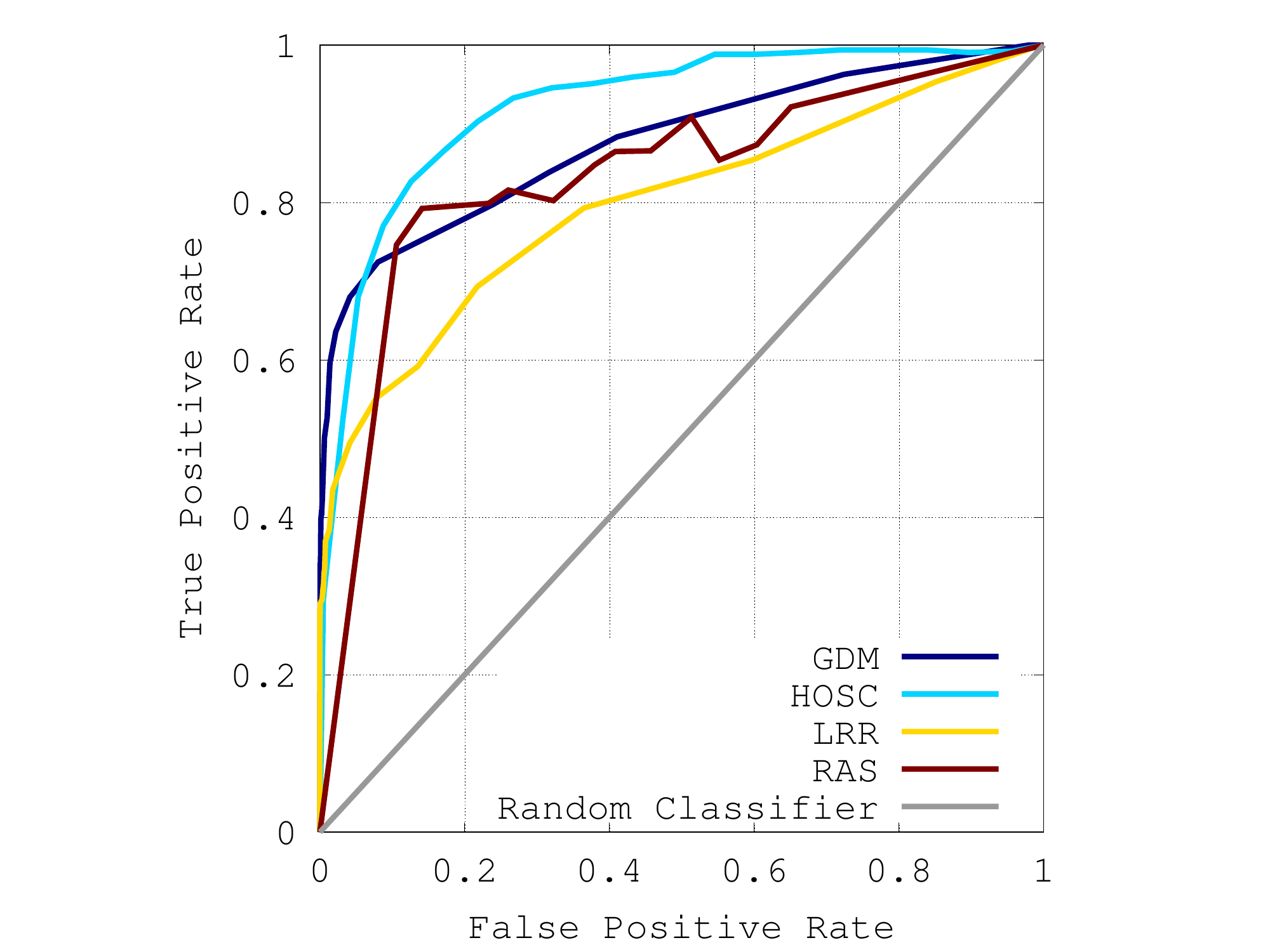}
\caption{\label{fig:ROC_Figure}The outlier detection performance of GDM (Model-Reassign) is compared against other motion segmentation methods on the outlier-free RAS database. (Color figure online)}
\end{figure}

The TPR and FPR for a robust segmentation algorithm cannot generally be controlled independently or arbitrarily. Thus, for a comparison of segmentation accuracy, one must select ``reasonable'' parameters for each method, which correspond to the same general region of ROC space. It should be understood that since the TPR and FPR cannot be controlled exactly for each method, any such comparison is inherently unfair, and by manipulating outlier parameters the results can be skewed somewhat in any direction.

For the purpose of fairly comparing GDM with other methods, we must select only one of the suggested outlier detection schemes for GDM (``GDM - Known Fraction'' or ``GDM Model-Reassign''). To effectively use ``GDM - Known Fraction'', one must either know roughly what fraction of his or her data are going to be outliers, or be in a situation where over-rejecting points as outliers is acceptable (you can then over-estimate the outlier fraction). Since this is not usually the case, we will consider the results of ``GDM Model-Reassign'' when comparing with other methods.

In Table \ref{Table:InlierMis-Rates} we present a file-by-file comparison of segmentation accuracy for the aforementioned methods using parameters that place the FPR of each method in the range of 0.01 to 0.08. Table \ref{Table:TPR and FPR for Comparison} reports the average TPR and FPR for each of these methods.

\begin{table}[H]
	\captionsetup{margin=10pt,labelfont=bf}
	\caption{True Positive Rate (TPR) and False Positive Rate (FPR) for each method in our segmentation comparison in Table \ref{Table:InlierMis-Rates}. GDM - Model Reassign, RAS, LRR, and HOSC were each tuned to achieve a false positive rate in the range of 0.01 to 0.08.}
	\centering
	\tabcolsep=0.20cm
	\begin{tabular}{ l l | c c |}
	\cline{3-4}
	& & \multicolumn{1}{|c|}{TPR}  & \multicolumn{1}{|c|}{FPR} \\ \hline
	\multicolumn{1}{|c|}{\multirow{6}{*}{\begin{sideways}\bf{Method}\end{sideways}}}
                             & \multicolumn{1}{|l|}{GDM - Model-Reassign}           & 0.56 & 0.01 \\ \cline{2-4}
	\multicolumn{1}{|c|}{}   & \multicolumn{1}{|l|}{GDM - Classic}                  & NA & NA \\ \cline{2-4}
     \multicolumn{1}{|c|}{}  & \multicolumn{1}{|l|}{RAS}                            & 0.74 & 0.08 \\ \cline{2-4}
	\multicolumn{1}{|c|}{}   & \multicolumn{1}{|l|}{LRR}                            & 0.49 & 0.04 \\ \cline{2-4}
	\multicolumn{1}{|c|}{}   & \multicolumn{1}{|l|}{HOSC}                            & 0.71 & 0.06 \\ \cline{2-4}
	\rowcolor{lightgray} \multicolumn{1}{|c|}{} & \multicolumn{1}{|l|}{GDM - clean} & NA & NA \\ \cline{2-4}
	\hline
	\end{tabular}
   	\label{Table:TPR and FPR for Comparison}
\end{table}

One can see from Table \ref{Table:InlierMis-Rates} that ``GDM Model-Reassign'' causes an overall improvement in segmentation accuracy (vs ``GDM - Classic'') in the presence of outliers. There were several files where the outliers cause the classic GDM method to misclassify large fractions of the data sets (files 4, 8, and 11 have inlier misclassification rates over 20\%). On these files the error rates of ``GDM Model-Reassign'' are dramatically lower. There are some files where the outlier detection framework appears to hurt performance, but in most of these cases the degradation is slight. The results for GDM are better in most cases (and on average) than the competing methods, although there are a few files where GDM is outperformed by a small margin. Unlike the strong outlier detection performance of HOSC discussed earlier, the segmentation capabilities of HOSC appear very intolerant to outliers (if even a few outliers slip through, segmentation performance suffers).


\section{Conclusions}
\label{sec:Conclusion}
We presented a new approach to 2-view motion segmentation, which is also a general method for HLM. Its development was motivated by the main obstacle of recovering multiple subspaces within the nonlinear embedding of point correspondences into $\R^9$; namely, the nonuniform distributions along subspaces (of unknown dimensions). The idea was to minimize a global quantity, i.e., global dimension. Unlike KSCC~\cite{AtevKSCC}, which also exploits global information, this approach does not make an a-priori assumption on the dimensions of the underlying subspaces. We formulated a fast method to minimize this global dimension, which we referred to as GDM. We demonstrated state-of-the-art results of GDM for 2-view motion segmentation.

We carefully explained the meaning of the two main parameters in our algorithm, $p$ and $\epsilon$, and the trade-offs they express. We gave a theoretical basis for selecting an appropriate value of $p$. Needless to say that these parameters are fixed throughout the paper. We described a preliminary theory which motivated the notion of global dimension, and we justified why it makes sense as an objective function in our application.

Finally, we presented an outlier detection/rejection framework for GDM. We explored two complimentary implementations of this framework, and we presented results demonstrating that it is competitive at handling outliers in this application.

\section{Appendix}
\subsection{Proof of Theorem \ref{thm:Empirical Dimension Properties}}	
We prove the four properties of the statement of the theorem.	
For simplicity we assume that $D<N$. That is, the number of data points is greater than the dimension of the ambient space. This is the usual case in many applications.\\[0.5cm]
Proof of Property 1:
Clearly, scaling all data vectors by $\alpha \ne 0$ results in scaling all the singular values of the corresponding data matrix by $\alpha$. Furthermore, this results in scaling by $\alpha$ both the numerator and denominator of the expression for the empirical dimension for any $\epsilon>0$. Therefore, the empirical dimension is invariant to this scaling.\\[0.5cm]
Proof of Property 2:
The singular values of a matrix (in particular the data matrix) are invariant to any orthogonal transformation of this matrix and thus the empirical dimension is invariant to such transformation.\\[0.5cm]
Proof of Property 3:
If $\{\bv_i\}_{i=1}^N$ are contained in a $d$-subspace, then since these form the columns of $\A$, $rank(\A) \le d$. Since $\U$ and $\V$ are orthogonal, $rank(\A)=rank(\bSigma)$. In particular, $\A$ has at most $d$ singular values. Let $\bsigma$ be the vector of singular values of $\A$, and let $1_{\bsigma}$ be the indicator vector of $\bsigma$\footnote{$1_{\bsigma}$ has a $1$ in each coordinate where $\bsigma$ has a non-zero element, and $0$'s in all other coordinates.}.

The generalized H\"{o}lder's Inequality \cite[pg. 10]{grafakos2004classical} states that if:
\begin{equation}
\label{eqn:Generalized Holders Inequality - Hypotheses}
p_1,p_2 \in (0,\infty] \;\; \text{ and } \;\; {1 \over p_1} + {1 \over p_2} = {1 \over r}
\end{equation}
then
\begin{equation}
\Vert f_1 f_2 \Vert_r \le \Vert f_1 \Vert_{p_1} \Vert f_2 \Vert_{p_2} \; \text{ for any functions $f_1$ and $f_2$.}
\end{equation}
To apply this result to vectors, we view them as functions over the set $\{ 1,2,...,D \}$ with counting measure.

Let $p_1=1$, $p_2={\epsilon \over {1-\epsilon}}$, $r=\epsilon$. Also let $f_1=1_{\bsigma}$, $f_2=\bsigma$. These values satisfy \eqref{eqn:Generalized Holders Inequality - Hypotheses}. We therefore get:
\begin{equation}
{{\Vert \bsigma \Vert_{\epsilon}} \over {\Vert \bsigma \Vert_{\epsilon \over {1-\epsilon}}}} \le \Vert 1_{\bsigma} \Vert_1 = (\text{\# of non-zero sing. values of $\A$}) \le d.
\end{equation}
Proof of Property 4:
By hypothesis, the data vectors $\{\bv_i\}_{i=1}^N$ are i.i.d. and sampled according to probability measure $\mu$, where $\mu$ is sub-Gaussian, non-degenerate, and spherically symmetric in a $d$-subspace of $\R^D$. We define the $n$th data matrix:
	$$\A_n = \left[
		\begin{tabular}{ccccc}
			$\uparrow$   & $\uparrow$   & $\uparrow$   &            & $\uparrow$   \\
			$\bv_1$       & $\bv_2$       & $\bv_3$       & $\cdots$   & $\bv_n$       \\
			$\downarrow$ & $\downarrow$ & $\downarrow$ &            & $\downarrow$ \\
		\end{tabular}
		\label{tab:}
	\right].$$
	Then $\bSigma_n := ({1 \over n}) \A_n \A_n^T$ is the $n$th sample covariance matrix of our data set. Also, let $\bv$ be a random variable with probability measure $\mu$. Then $\bSigma := E[\bv \bv^T]$ is the covariance matrix of the distribution. A consequence of $\mu$ being spherically symmetric in a $d$-subspace is that after an appropriate rotation of space, $\bSigma$ is diagonal with a fixed constant in $d$ of its diagonal entries and $0$ in all other locations. We are trying to prove a result about empirical dimension, which is scale invariant and invariant under rotations of space. Because of these two properties we can assume that the appropriate rotation and scaling has been done so that $\bSigma$ is diagonal with value $1$ in $d$ diagonal entries and $0$ in all others. Without any loss of generality, we assume that the first $d$ diagonal entries are the non-zero ones.
	
Let $\bsigma_n = \left( \sigma_{n,1}, \sigma_{n,2}, ..., \sigma_{n,D} \right)^T$, $n \ge D$,	denote the vector of singular values of the matrix $\A_n$. Our first task will be to show that ${\bsigma_n \over \sqrt{n}}$ converges in probability (as $n \rightarrow \infty$) to the vector:
\begin{equation}
\label{eqn:limiting vector of s.v's}
( \underbrace{ 1, 1 ,...,1}_d, 0,...,0 )^T.
\end{equation}

To accomplish our task, we will first relate $\bsigma_n$ to the vector of singular values of $\bSigma_n$, and then use a result showing that $\bSigma_n$ converges to $\bSigma$ as $n \rightarrow \infty$.

%
It is clear that the vector of singular values of $\bSigma_n$, which we will denote by $\bpsi$, is given by:
\begin{equation}
\bpsi = {1 \over n} \left( \sigma_{n,1}^2, \sigma_{n,2}^2, ..., \sigma_{n,D}^2 \right)^T.
\end{equation}

Next, we will need the following result regarding covariance estimation. This is Corollary 5.50 of \cite{vershynin_book}, adapted to be consistent with our notation.
	
	\bigskip
	Lemma 1 (Covariance Estimation): Consider a sub-Gaussian distribution in $\R^D$ with covariance matrix $\bSigma$. Let $\gamma \in (0,1)$, and $t \ge 1$. If $n > C(t/ \gamma)^2 D$, then with probability at least $1-2e^{-t^2 D}$, $\Vert \bSigma_n - \bSigma \Vert_2 \le \gamma$, where $\Vert \cdot \Vert_2$ denotes the spectral norm (i.e., largest singular value of the matrix). The constant $C$ depends only on the sub-Gaussian norm of the distribution.
	
	\bigskip
	In our problem, we are applying this lemma to the distribution $\mu$. Let $\gamma \in (0,1)$ be given. If
	\begin{equation} \label{eq:Lemma1Hypothesis}
	n > C(t/ \gamma)^2 D,
	\end{equation}
	then $\Vert \bSigma_n-\bSigma \Vert_2 \le \gamma$ with probability at least $1-2e^{-t^2 D}$. The $2$-norm of the difference of two matrices bounds the differences of their individual singular values. We will use the following result to make this precise:
	
	\bigskip
	Lemma 2 \cite{bhatia:1997}: Let $\sigma_i(\bullet)$ denote the $i$th largest singular value of an arbitrary $m$-by-$n$ matrix. Then: $| \sigma_i(\mathbf{B}+\mathbf{E})-\sigma_i(\mathbf{B})| \le \Vert \mathbf{E} \Vert_2$, for each $i$.
	
\bigskip
Because	$\bSigma$ is diagonal with only values $1$ and $0$ on the diagonal, the singular values of $\bSigma$ are simply these diagonal values. We will use $\mathbf{1}_{i \in 1:d}$ to denote the $i$'th singular value of $\bSigma$.
	
Setting $\mathbf{B} = \bSigma_n$ and $\mathbf{E} = \bSigma - \bSigma_n$, in lemma 2 we get: $\Vert \bSigma_n-\bSigma \Vert_2 \le \gamma \Rightarrow |(1/n)\sigma_{n,i}^2 - \mathbf{1}_{i \in 1:d}| \le \Vert \bSigma_n-\bSigma \Vert_2 \le \gamma$, for each $i$. This implies that:
\begin{equation}
{\sigma_{n,i} \over \sqrt{n}} \in \left\{
	\begin{array}{ll}
		\left[ \sqrt{{1} - \gamma}, \sqrt{{1} + \gamma} \right],  & \mbox{ if } i \le d; \\
		\left[0, \sqrt{\gamma} \right], & \mbox{ if } i > d.
	\end{array}
\right.
\end{equation}

Notice that as $\gamma \rightarrow 0$, ${\sigma_{n,i} \over \sqrt{n}}$ approaches $\mathbf{1}_{i \in 1:d}$. Specifically, for any desired tolerance, $\eta >0$, and any desired certainty, $\xi$, $n$ can be chosen large enough that with probability greater than $\xi$, $\left|\mathbf{1}_{i \in 1:d}-{\sigma_{n,i} \over \sqrt{n}} \right| < \eta$, simultaneously for each $i$. It follows from this that the vector ${\bsigma_n \over \sqrt{n}}$ converges in probability to \eqref{eqn:limiting vector of s.v's} as $n \rightarrow \infty$.

Finally, $\hat{d}_{\epsilon,n} = {{\Vert \bsigma_n \Vert_{\epsilon}} \over {\Vert \bsigma_n \Vert_{\epsilon \over {1-\epsilon}}}} = {\left( 1 \over \sqrt{n}\right){\Vert \bsigma_n \Vert_{\epsilon}} \over {\left( 1 \over \sqrt{n}\right) \Vert \bsigma_n \Vert_{\epsilon \over {1-\epsilon}}}} = {{\Vert {\bsigma_n \over \sqrt{n}} \Vert_{\epsilon}} \over {\Vert {\bsigma_n \over \sqrt{n}} \Vert_{\epsilon \over {1-\epsilon}}}}$. Thus, $\hat{d}_{\epsilon,n}$ is a continuous function of the vector ${\bsigma_n \over \sqrt{n}}$. Hence, since ${\bsigma_n \over \sqrt{n}}$ converges to $\mathbf{1}_{i \in 1:d}$ as $n \rightarrow \infty$, $\hat{d}_{\epsilon,n}$ converges in probability to
	
\begin{equation}
\left( {{\Vert (1,1,...,1,0,...,0) \Vert_\epsilon} \over {\Vert (1,1,...,1,0,...,0) \Vert_{\left({{\epsilon} \over {1-\epsilon}}\right)}}} \right) = {{d^{{1} \over {\epsilon}}} \over {d^{{1-\epsilon} \over {\epsilon}}}} = d^{{{1} \over {\epsilon}} - {{1-\epsilon} \over {\epsilon}}} = d.
\end{equation}

\subsection{Proof of Theorem \ref{THM:Global Dimension Theorem}}
Recall that $\Pi_{Nat}$ denotes the natural partition of the data set. First, we notice that $GD(\Pi_{Nat}) = \Vert (d_1,d_2,...,d_K) \Vert_p$, where $d_k$ is the true dimension of set $k$ of the partition. Notice that $d_k$ cannot exceed $d$ since $\mu_k$ is supported by $L_k$, a $d$-subspace. Furthermore, since $\mu_k$ does not concentrate mass on subspaces it is a probability 0 event that all $N_k$ points from $L_k$ exist in a proper subspace of $L_k$. Thus, for the natural partition, $d_k$ is almost surely $d$, for each $k$. Hence, $GD(\Pi_{Nat})$ is almost surely $\Vert (d,d,...,d) \Vert_p = \left( K d^p\right)^{1/p} = K^{1/p} d$.

Next, we will find a lower bound for the global dimension of any non-natural partition of the data, and show that if $p$ meets the hypothesis criteria, the lower bound we get is greater than $K^{1/p} d = GD(\Pi_{Nat})$. To accomplish this we need the following lemma.
\begin{lemma}
\label{lemma:Main Lemma}
If $\Pi \ne \Pi_{Nat}$ then $\Pi$ almost surely has one set with dimension at least $d+1$.
\end{lemma}
Before proving the lemma, observe that a consequence is that if $\Pi \ne \Pi_{Nat}$, then with probability 1:
\begin{equation}
GD(\Pi) \ge \Vert (?,...,?,d+1,?,...,?) \Vert_p \ge d+1.
\end{equation}
Then, from our hypothesis:
\begin{align}
p > &ln(K)/(ln(d+1)-ln(d)) \notag \\
\Longrightarrow & \; \left( {{d+1} \over d} \right)^p > K \notag \\
\Longrightarrow & \; d+1 > K^{1/p} d.
\end{align}
Hence,
\begin{equation}
GD(\Pi) \ge d+1 > K^{1/p} d = GD(\Pi_{Nat}).
\end{equation}
Thus, if we show Lemma \ref{lemma:Main Lemma}, the proof of the theorem follows. To prove Lemma \ref{lemma:Main Lemma} we require an a simpler lemma:
\begin{lemma}
\label{lemma:Minor Lemma}
If a set $Q$ in $\Pi$ has fewer than $d$ points from a subspace $L_i$, then either $Q$ has dimension at least $d+1$ or adding another point from $L_i$ to $Q$ (an R.V. $X$ with probability measure $\mu_i$, independent from all other samples) will almost surely increase the dimension of $Q$ by 1.
\end{lemma}
\begin{proof}
If $\Dim(Q) \le d$ then $Q$ has dimension strictly less than the ambient space ($\R^D$). Observe that $\Span (Q)$ is a linear subspace of $\R^D$, which a.s. does not contain $L_i$. We cannot have proper containment since $\Dim(L_i)=d \ge \Dim(Q)$. Also, we have fewer than $d$ points from $L_i$ in $Q$, and each other point in $Q$ lies in $L_i$ with probability 0 (All $\mu_i$ do not concentrate mass on subspaces). Thus, $\Span(Q)$ a.s. does not equal $L_i$.

Therefore, if we intersect $L_i$ with $\Span(Q)$ we get a proper subspace of $L_i$; call it $\bar{L}$. We note that $\mu_i(\bar{L})=0$ since $\mu_i$ does not concentrate on subspaces. Thus, since $X$ has probability measure $\mu_i$, $X$ a.s. lies outside the intersection of $L_i$ and $\Span(Q)$. It follows that if we add $X$ to $Q$, the dimension of $Q$ a.s. increases by 1. \QED
\end{proof}

Now we prove Lemma \ref{lemma:Main Lemma}. We will assume all sets in $\Pi$ have dimension less than $d+1$ and pursue a contradiction. By hypothesis, our set $\{\bv_n\}_{n=1}^N$ contains at least $d+1$ points from each subspace $L_i$. Since $\Pi \ne \Pi_{Nat}$, there is some subspace $L^*$ whose points are assigned to $2$ or more distinct sets in $\Pi$. Let $\bv^*$ be a point from $L^*$. Now, choose $d$ points from each $L_i$ and denote this collection of $Kd$ points $\{y_1,y_2,...,y_{Kd}\}$. When making this selection, ensure that $v^*$ is not chosen and that of the points selected from $L^*$, not all of them are assigned to the same set in $\Pi$ as $\bv^*$. Notice that $\Pi$ induces a partition on $\{y_1,y_2,...,y_{Kd}\}$.

Select any point $y_i$ and remove it from the set $\{y_1,y_2,...,y_{Kd}\}$. Since we are assuming that each set in $\Pi$ has dimension less than $d+1$, Lemma 2 implies that the set in $\Pi$ to which $y_i$ belongs will have its dimension decrease by $1$. Now select another point $y_j$ and remove it. Lemma 2 still applies and so the set to which $y_j$ belonged will have its dimension decrease by 1. We can repeat this until all $Kd$ points have been removed. Since each removal decreases the dimension of some set in $\Pi$ by $1$ it follows that before any removals the sum of the dimensions of all sets in $\Pi$ was at least $Kd$. Since each of the $K$ sets in $\Pi$ had dimension $d$ or less, we conclude that in fact each set must have had dimension exactly $d$.

Now, consider our set $\{y_1,y_2,...,y_{Kd}\}$ and add in $v^*$. By our choice of $v^*$, Lemma $2$ implies that its addition a.s. increases the dimension of its target set in $\Pi$ by $1$ (to $d+1$). Adding in all remaining points from $\{\bv_n\}_{n=1}^N$ will only increase the dimensions of the sets in $\Pi$. Thus, we almost surely have a set of dimension at least $d+1$ in $\Pi$, contradicting our hypothesis.\QED

\subsection{Proof of Theorem 3}
Recall that the soft partition is stored in a membership matrix $\M$. Specifically, the $(k,n)$'th element of $\M$, denoted $m_k^n$, holds the ``probability'' that vector $\bv_n$ belongs to cluster $k$. Thus, each column of $\M$ forms a probability vector.

Hence, global dimension is a real-valued function of the matrix $\M$. We will think of the membership matrix as being vectorized, so that the domain of optimization can be thought of as a subset of $\R^{NK}$. However, we will not explicitly vectorize the membership matrix. Thus, when we talk about the \emph{gradient} of global dimension, we are referring to another $K$-by-$N$ matrix, where the $(k,n)$'th element is the derivative of global dimension w.r.t.~$m_k^n$.

To differentiate global dimension we must be able to differentiate the singular values of a matrix w.r.t.~each element of that matrix. A treatment of this is available in \cite{PapadopouloSVD}.

To begin, recall the definition of $GD$:
\begin{equation}
\label{eqn:DefOfG}
GD = \left \Vert
\begin{tabular}{c}
$\hat{d}_\epsilon^1$ \\ $\hat{d}_\epsilon^2$ \\ $\vdots$ \\ $\hat{d}_\epsilon^K$
\end{tabular}
\right \Vert_p = \left( (\hat{d}_\epsilon^1)^p + (\hat{d}_\epsilon^2)^p + ... + (\hat{d}_\epsilon^K)^p \right)^{1/p}.
\end{equation}
We will denote the thin SVD (only $D$ columns of $\U$ and $\V$ are used) of $\A_k$:
\begin{equation}
\A_k = \U_k \bSigma_k {\V_k}^T.
\end{equation}
Also, we will let $\sigma_j^i$ refer to the $(j,j)$'th element of $\bSigma_i$.
Then, using the chain rule:
\begin{equation}
\label{eqn:pG/pm_k^n}
\frac{\partial GD}{\partial m_k^n} = \frac{\partial GD}{\partial \hat{d}_\epsilon^1} \frac{\partial \hat{d}_\epsilon^1}{\partial m_k^n} + \frac{\partial GD}{\partial \hat{d}_\epsilon^2} \frac{\partial \hat{d}_\epsilon^2}{\partial m_k^n} + ... + \frac{\partial GD}{\partial \hat{d}_\epsilon^K} \frac{\partial \hat{d}_\epsilon^K}{\partial m_k^n}.
\end{equation}
From \eqref{eqn:DefOfG} we can compute $\frac{\partial GD}{\partial \hat{d}_\epsilon^i}$ rather easily:
\begin{align}
\notag \frac{\partial GD}{\partial \hat{d}_\epsilon^i} =& \frac{1}{p} \left( (\hat{d}_\epsilon^1)^p + (\hat{d}_\epsilon^2)^p + ... + (\hat{d}_\epsilon^K)^p \right)^{\frac{1}{p}-1} p (\hat{d}_\epsilon^i)^{p-1}\\
 =& (\hat{d}_\epsilon^i)^{p-1} \left( (\hat{d}_\epsilon^1)^p + (\hat{d}_\epsilon^2)^p + ... + (\hat{d}_\epsilon^K)^p \right)^{\frac{1}{p}-1}.
\end{align}
Next, we expand the other components of \eqref{eqn:pG/pm_k^n}:
\begin{equation}
\label{eqn:partiald_i/partialm_k^n}
\frac{\partial \hat{d}_\epsilon^i}{\partial m_k^n} = \sum_{j=1}^D \frac{\partial \hat{d}_\epsilon^i}{\partial \sigma_j^i} \frac{\partial \sigma_j^i}{\partial m_k^n}.
\end{equation}

We now use the definition of $\hat{d}_\epsilon^i$ to compute the first factor of each term as follows:
\begin{align}
\notag \frac{\partial \hat{d}_\epsilon^i}{\partial \sigma_j^i} =&
{{\Vert \bsigma_i \Vert_\delta {{\partial} \over {\partial \sigma_j^i}} \left( \left( (\sigma_1^i)^\epsilon + ... + (\sigma_D^i)^\epsilon \right)^{1/\epsilon} \right)} \over {\Vert \bsigma_i \Vert_\delta^2}} -\\
\notag & {{\Vert \bsigma_i \Vert_\epsilon {{\partial} \over {\partial \sigma_j^i}} \left( \left( (\sigma_1^i)^\delta + ... + (\sigma_D^i)^\delta \right)^{1/\delta} \right)} \over {\Vert \bsigma_i \Vert_\delta^2}} \\
\notag =& {{\Vert \bsigma_i \Vert_\delta \left( (\sigma_1^i)^\epsilon + ... + (\sigma_D^i)^\epsilon \right)^{{1-\epsilon} \over {\epsilon}} \left( \sigma_j^i \right)^{\epsilon-1}} \over {\Vert \bsigma_i \Vert_\delta^2}} -\\
\notag & {{\Vert \bsigma_i \Vert_\epsilon \left( (\sigma_1^i)^\delta + ... + (\sigma_D^i)^\delta \right)^{{1-\delta} \over {\delta}} \left( \sigma_j^i \right)^{\delta-1}} \over {\Vert \bsigma_i \Vert_\delta^2}} \\
\notag =& \left( {{1} \over {\Vert \bsigma_i \Vert_\delta^2}} \right) \Vert \bsigma_i \Vert_\delta \Vert \bsigma_i \Vert_\epsilon^{1-\epsilon} \left( \sigma_j^i \right)^{\epsilon-1} -\\
\notag & \left( {{1} \over {\Vert \bsigma_i \Vert_\delta^2}} \right) \Vert \bsigma_i \Vert_\epsilon \Vert \bsigma_i \Vert_\delta^{1-\delta} \left( \sigma_j^i \right)^{\delta-1} \\
=& C_1^i \left( \sigma_j^i \right)^{\epsilon-1} - C_2^i \left( \sigma_j^i \right)^{\delta-1},
\end{align}
where
\begin{equation}
C_1^i = \left( {{\Vert \bsigma_i \Vert_\epsilon^{1-\epsilon} \Vert \bsigma_i \Vert_\delta} \over {\Vert \bsigma_i \Vert_\delta^2}} \right), \; \; \; \; \;
C_2^i = \left( {{\Vert \bsigma_i \Vert_\epsilon \Vert \bsigma_i \Vert_\delta^{1-\delta}} \over {\Vert \bsigma_i \Vert_\delta^2}} \right).
\end{equation}

Next, we must evaluate the second factor in each term of \eqref{eqn:partiald_i/partialm_k^n}. Recall that $\sigma_j^i$ is the $j$'th largest singular value of the matrix $\A_i$. To achieve the next step, we must observe that each singular value of $\A_i$ depends, in general, on each element of the matrix $\A_i$. We can then compute the derivative of each element of the matrix $\A_i$ w.r.t.~each membership variable, $m_k^n$. We will denote the $(\alpha,\beta)$'th element of the matrix $\A_i$ by ${\A_i}_{(\alpha,\beta)}$. Using the chain rule:
\begin{equation}
\label{eqn:partialSigma_j^i/partialm_k^n}
\frac{\partial \sigma_j^i}{\partial m_k^n} = \sum_{\beta = 1}^N \sum_{\alpha=1}^D \frac{\partial \sigma_j^i}{\partial {\A_i}_{(\alpha,\beta)}} \frac{\partial {\A_i}_{(\alpha,\beta)}}{m_k^n}.
\end{equation}

A powerful result \cite[eqn. 7]{PapadopouloSVD} allows us to express the partial derivative of each singular value, $\sigma_j^i$, w.r.t.~a given matrix element in terms of the already-known SVD of $\A_i$:
\begin{equation}
\frac{\partial \sigma_j^i}{\partial {\A_i}_{(\alpha,\beta)}} = \U_{i(\alpha,j)} \V_{i(\beta,j)}.
\end{equation}

The second factor in each term of \eqref{eqn:partialSigma_j^i/partialm_k^n} can be evaluated directly from the definition of $\A_k$:
\begin{equation}
\frac{\partial {\A_i}_{(\alpha,\beta)}}{\partial m_k^n} =
\begin{cases}
0, & \text{ if $n \neq \beta$ or if $i \neq k$;} \\
\bv_n \cdot \hat{\mathbf{e}_{\alpha}}, & \text{ if $n = \beta$ and $i = k$,}
\end{cases}
\end{equation}
where $\hat{\mathbf{e}_{\alpha}}$ denotes the $\alpha$'th standard basis vector (1 in position $\alpha$ and $0$'s everywhere else).

We are now in a position to work backwards and construct the partial derivative of $GD$ w.r.t.~$m_k^n$. In what follows, $\delta_{ik}$ is equal to $1$ if $i=k$ and is $0$ otherwise (this is not to be confused with the un-subscripted $\delta$, which is shorthand for $\epsilon/(1-\epsilon)$). Also, for notational convenience, we use Matlab notation to represent a row or column of a matrix ($B_{(w,:)}$ and $B_{(:,w)}$, respectively). We fist compute $\partial \sigma_j^i / \partial m_k^n $ as follows:
{\small
\begin{equation}
\begin{split}
\frac{\partial \sigma_j^i}{\partial m_k^n} = & \sum_{\alpha=1}^D \frac{\partial \sigma_j^i}{\partial {\A_i}_{(\alpha,n)}} \frac{\partial {\A_i}_{(\alpha,n)}}{m_k^n} \\
= & \sum_{\alpha=1}^D \U_{i(\alpha,j)} \V_{i(n,j)} \left( \bv_n \cdot \hat{\mathbf{e}_{\alpha}} \right) \delta_{ik} \\
= & \left[
\begin{tabular}{c}
$\U_{i(1,j)} \V_{i(n,j)}$ \\
$\U_{i(2,j)} \V_{i(n,j)}$ \\
\vdots \\
$\U_{i(D,j)} \V_{i(n,j)}$
\end{tabular}
\right] \cdot \bv_n \delta_{ik} =
\V_{i(n,j)}
\left[
\begin{tabular}{c}
$\U_{i(1,j)}$ \\
$\U_{i(2,j)}$ \\
\vdots \\
$\U_{i(D,j)}$
\end{tabular}
\right] \cdot \bv_n
\delta_{ik} \\
= & \V_{i(n,j)} \left( \U_{i(:,j)} \cdot \bv_n \right)  \delta_{ik}.
\end{split}
\end{equation}}
Then from \eqref{eqn:partiald_i/partialm_k^n}, we get
{\small
\begin{equation}
\begin{split}
\frac{\partial \hat{d}_\epsilon^i}{\partial m_k^n} = & \sum_{j=1}^D \frac{\partial \hat{d}_\epsilon^i}{\partial \sigma_j^i} \frac{\partial \sigma_j^i}{\partial m_k^n} \\
= & \sum_{j=1}^D \left( C_1^i \left( \sigma_j^i \right)^{\epsilon-1} - C_2^i \left( \sigma_j^i \right)^{\delta-1} \right) \V_{i(n,j)} \left( \U_{i(:,j)} \cdot \bv_n \right) \delta_{ik}.
\end{split}
\end{equation}}
Now we can write:
\begin{equation} \begin{split}
\label{eqn:partials_i/partialm_k^n-2}
\frac{\partial \hat{d}_\epsilon^i}{\partial m_k^n} =& C_1^i \left( \sum_{j=1}^D \left( \sigma_j^i \right)^{\epsilon-1} \V_{i(n,j)} \left( \U_{i(:,j)} \cdot \bv_n \right) \delta_{ik} \right) -\\
& C_2^i \left( \sum_{j=1}^D \left( \sigma_j^i \right)^{\delta-1} \V_{i(n,j)} \left( \U_{i(:,j)} \cdot \bv_n \right) \delta_{ik} \right).
\end{split} \end{equation}

We now simplify the components of \eqref{eqn:partials_i/partialm_k^n-2}. After some manipulation, and using the notation
\begin{equation}
\left( \bSigma_i \right)^{\epsilon-1} = \left[
\begin{tabular}{c c c}
$(\sigma_1^i)^{\epsilon-1}$ & $0$      & $0$ \\
$0$                         & $\ddots$ & $0$ \\
$0$                         & $0$      & $(\sigma_D^i)^{\epsilon-1}$
\end{tabular}
\right],
\end{equation}
we can write
{\small
\begin{equation}
\label{eqn:firstSimplification}
\begin{split}
\sum_{j=1}^D & \left( \sigma_j^i \right)^{\epsilon-1} \V_{i(n,j)} \left( \U_{i(:,j)} \cdot \bv_n \right) =\\
& \left[ \left( \sigma_1^i \right)^{\epsilon-1} \V_{i(n,1)}, ..., \left( \sigma_D^i \right)^{\epsilon-1} \V_{i(n,D)} \right] \left[
\begin{tabular}{c}
$\U_{i(:,1)} \cdot \bv_n$ \\
$\U_{i(:,2)} \cdot \bv_n$ \\
\vdots \\
$\U_{i(:,D)} \cdot \bv_n$
\end{tabular}
\right] \\
& = \V_{i(n,:)} \left( \bSigma_i \right)^{\epsilon-1} \left( \U_i \right)^T \bv_n.
\end{split}
\end{equation}
}
Similarly, we can simplify part of the second term of \eqref{eqn:partials_i/partialm_k^n-2}:

\begin{equation}
\label{eqn:secondSimplification}
\sum_{j=1}^D \left( \sigma_j^i \right)^{\delta-1} \V_{i(n,j)} \left( \U_{i(:,j)} \cdot \bv_n \right) = \V_{i(n,:)} \left( \bSigma_i \right)^{\delta-1} \left( \U_i \right)^T \bv_n.
\end{equation}
Substituting \eqref{eqn:firstSimplification} and \eqref{eqn:secondSimplification} into \eqref{eqn:partials_i/partialm_k^n-2} we get
\begin{align*}
\frac{\partial \hat{d}_\epsilon^i}{\partial m_k^n} = &\left[ C_1^i \left( \V_{i(n,:)} \left( \bSigma_i \right)^{\epsilon-1} \left( \U_i \right)^T \bv_n \right) -\right.\\
& \;\; \left. C_2^i \left( \V_{i(n,:)} \left( \bSigma_i \right)^{\delta-1} \left( \U_i \right)^T \bv_n \right) \right] \delta_{ik}.
\end{align*}
With this expression we are ready to evaluate \eqref{eqn:pG/pm_k^n} as follows:
{\small
\begin{align}
\label{eqn:Unsimplified partial GD partial m_k^n}
&{{\partial GD} \over {\partial m_k^n}} = \sum_{i=1}^K {{\partial GD} \over {\partial \hat{d}_\epsilon^i}} {{\partial \hat{d}_\epsilon^i} \over {\partial m_k^n}} \notag \\
=& \sum_{i=1}^K (\hat{d}_\epsilon^i)^{p-1} \left( (\hat{d}_\epsilon^1)^p + ... + (\hat{d}_\epsilon^K)^p \right)^{\frac{1}{p}-1} \delta_{ik} \cdot  \notag \\
& \; \left[ C_1^i \left( \V_{i(n,:)} \left( \bSigma_i \right)^{\epsilon-1} \left( \U_i \right)^T \bv_n \right) - C_2^i \left( \V_{i(n,:)} \left( \bSigma_i \right)^{\delta-1} \left( \U_i \right)^T \bv_n \right) \right]  \notag \\
=& (\hat{d}_\epsilon^k)^{p-1} \left( (\hat{d}_\epsilon^1)^p + ... + (\hat{d}_\epsilon^K)^p \right)^{\frac{1}{p}-1} \cdot  \notag \\
& \; \left[ C_1^k \V_{k(n,:)} \left( \bSigma_k \right)^{\epsilon-1} \left( \U_k \right)^T \bv_n - C_2^k \V_{k(n,:)} \left( \bSigma_k \right)^{\delta-1} \left( \U_k \right)^T \bv_n \right]  \notag \\
=& (\hat{d}_\epsilon^k)^{p-1} \Vert \left( \hat{d}_\epsilon^1, ..., \hat{d}_\epsilon^K \right) \Vert_p^{1-p} \cdot  \notag \\
& \; \left[ \V_{k(n,:)} \left( C_1^k \left( \bSigma_k \right)^{\epsilon-1} - C_2^k \left( \bSigma_k \right)^{\delta-1} \right) \left( \U_k \right)^T \bv_n \right]  \notag \\
=& (\hat{d}_\epsilon^k)^{p-1} \Vert \left( \hat{d}_\epsilon^1, ..., \hat{d}_\epsilon^K \right) \Vert_p^{1-p} \V_{k(n,:)} \D_k \left( \U_k \right)^T \bv_n,
\end{align}
}
where
\begin{equation}
\D_k = \left( C_1^k \left( \bSigma_k \right)^{\epsilon-1} - C_2^k \left( \bSigma_k \right)^{\delta-1} \right).
\end{equation}
We re-write \eqref{eqn:Unsimplified partial GD partial m_k^n} as follows:
\begin{equation}
{{\partial GD} \over {\partial m_k^n}} = \V_{k(n,:)} \left( (\hat{d}_\epsilon^k)^{p-1} \Vert \left( \hat{d}_\epsilon^1, ..., \hat{d}_\epsilon^K \right) \Vert_p^{1-p} \D_k \left( \U_k \right)^T \right) \A_{(:,n)}.
\end{equation}

\subsection{Experiment Setup}
\label{sec:Experiment Setup}
For our comparison on the outlier-free RAS database, we include the following methods: GDM, SCC~\cite{spectral_applied} from
www.math.umn.edu/$\sim$lerman/scc, MAPA~\cite{mapa} from www.math.duke.edu/$\sim$glchen/mapa.html, SSC
\cite{ssc09} (version 1.0 based on CVX) from www.vision.jhu.edu/code, SLBF (\& SLBF-MS)~\cite{LBF_journal12} from www.math.umn.edu/$\sim$lerm-
an/lbf, LRR~\cite{lrr_short} from sites.google.com/site/guangcanliu, RAS~\cite{rao_ijcv} (obtained directly from the authors),
and HOSC~\cite{higher-order} from www.math.duke.edu/$\sim$glchen/hosc.html. For each method in our comparisons (outlier-free and our tests with outliers) the implementation of each algorithm is that of the original authors. Most of these codes were found on the respective authors' websites, although some codes were obtained from the authors directly when they could not be found online. As a matter of good testing methodology, we ran each method 10 times on each file. This is because we want to avoid capturing any fluke occurrences of any method, but instead seek the ``usual case'' results (this is important for repeatability of the results). Of the 10 runs for a given file and method, the median error is reported. For deterministic methods, we get the same exact results for each run. GDM involves randomness, but the average standard deviation of the misclassification errors was $0.73\%$, meaning that it behaved very consistently in the experiment. GDM was run with $n_1=10$. The other parameters ($\epsilon$ and $p$) are fixed throughout all experiments and are addressed earlier. SCC was run with $d=3$ for the linearly embedded data (this was found to give the best results), and $d=7$ for the nonlinearly embedded data (as recommended in~\cite{AtevKSCC}). MAPA was run without any special parameters. SSC was run with no data projection (because of the low ambient dimension to start with), the affine constraint enabled (we tried it both ways and this gave better results), optimization method = ``Lasso'' (Default for authors code), and parameter lambda = 0.001 (found through trial and error). SLBF was run with $d=3$ for the linearly embedded data and $d=6$ for the nonlinearly embedded data and $\sigma$ was set to $20,000$ for both cases ($d$ and $\sigma$ were selected by trial and error to give the best results). LRR was run with $\lambda=100$ for the linear case and $\lambda=10000$ for the non-linear case (these seemed to give the best results). RAS proved rather sensitive to its main parameter (``angleTolerance''), and no single value gave good across-the-board results. We ran with all default parameters and many other combinations. The results presented were generated using $\text{angleTolerance}=0.22$ and $\text{boundaryThreshold}=5$, as this combination gave the best results from our tests (better than the algorithms defaults). HOSC was run with $\eta$ automatically selected by the algorithm from the range $[0.0001, 0.1]$. The parameter ``knn'' was set to 20, and the default ``heat'' kernel was chosen. The algorithm was tried with $d$ set to 2 and 3. Both of these cases are presented. $d=2$ gave better results, but the authors of HOSC argue for using $d=3$ in this setting.

For our comparison on the outlier-free Hopkins 155 database, the algorithms that were selected for the comparison were run once on each of the 155 data files. The mean and median performance for each category is reported. GDM was run with $n_1=30$ to improve reliability. All other parameters were left fixed, and (as before) the non-linearly embedded data was used. Each competing 2-view method was run on the non-linearly embedded data with the same parameters that gave the best performance on the RAS database. The competing n-view methods have their parameters given in the results tables.

For our outlier comparison on the corrupted RAS database, we ran GDM with $n_1=30$ (same as for the Hopkins 155 database). For the naive approach, we used $\alpha=0.02$. For ``GDM - Known Fraction'' we rejected 20\% of the dataset. For ``GDM - Model Reassign'' we used $\kappa=0.05$. ``GDM - Classic'' was the same algorithm as in the outlier-free comparisons and so had no extra parameters. RAS was run with $\text{angleTolerance}=0.22$ and $\text{boundaryThreshold}=5$ (same as in the outlier-free tests). We ran LRR with $\lambda=0.1$, and $\text{outlierThreshold}=0.138$ (these gave the best results of the combinations we tried). HOSC was run with $d=2$ (which gave the best results in the outlier-free case) and $\alpha=0.11$.

The code for GDM can be found on our supplemental webpage. For each of the algorithms used in our comparisons, we have made an effort to provide (on the supplemental webpage) the code or a link to where the code can be found.

%
%


\begin{acknowledgements}
This work was supported by NSF grants DMS-09-15064 and DMS-09-56072. GL was partially supported by the IMA during their annual program on the mathematics of information (2011-2012) and BP benefited from participating in parts of this program and even presented an initial version of this work at an IMA seminar in Spring 2012. We thank the anonymous reviewers for their thoughtful comments and Tom Lou for his helpful suggestions in regards to our algorithm for minimizing global dimension. A very preliminary version of this work was submitted to CVPR 2012, we thank one of the anonymous reviewers for some insightful comments that made us modify the GDM algorithm and its theoretical support.
\end{acknowledgements}

\bibliographystyle{spmpsci}
\bibliography{refs_10_31_13}

\end{document}